\documentclass[10 pt,journal,onecolumn]{IEEEtran}

\usepackage{adjustbox}
\usepackage{amsmath,amssymb,mathtools,amsthm}
\usepackage{physics}
\usepackage{empheq}
\usepackage{graphicx}
\usepackage{booktabs}
\usepackage{algorithm}
\usepackage{algpseudocode}
\usepackage{multirow}
\usepackage{cite}
\usepackage{bbm}
\usepackage{authblk}
\usepackage{siunitx}
\usepackage{hyphenat}
\usepackage{tikz}
\usetikzlibrary{spy,calc,math,intersections}
\usepackage{pgfplots}
\usepgfplotslibrary{groupplots}
\pgfplotsset{compat=1.14}
\usepackage[caption=false]{subfig}
\usepackage{pgfplotstable}
\usepackage[nolist]{acronym}
\usepackage{hyperref}
\usepackage{xcolor}
\newif\ifshowchanges
\showchangesfalse
\DeclareRobustCommand{\rev}[1]{\begingroup\ifshowchanges\color{blue}\fi #1\endgroup}

\ifshowchanges
\hypersetup{pdfsubject={Minimal revision: changes from the original draft are marked in blue; experiment-log checks remain with the authors.}}
\fi

\makeatletter
\newcommand{\markrevisedreference}[1]{\expandafter\def\csname BOLT@revised@#1\endcsname{}}
\newcommand{\BOLT@referencecolor}[1]{%
  \ifshowchanges
    \@ifundefined{BOLT@revised@#1}{}{\color{blue}}%
  \fi}
\let\BOLT@bibitem\@bibitem
\renewcommand{\@bibitem}[1]{\normalcolor\BOLT@bibitem{#1}\BOLT@referencecolor{#1}}
\let\BOLT@lbibitem\@lbibitem
\def\@lbibitem[#1]#2{\normalcolor\BOLT@lbibitem[#1]{#2}\BOLT@referencecolor{#2}}
\makeatother
\markrevisedreference{nguyen2009surrogate}
\markrevisedreference{reid2011information}
\markrevisedreference{nguyen2010estimating}
\markrevisedreference{ghosh2017robust}
\markrevisedreference{charoenphakdee2019symmetric}

\DeclareMathOperator{\Var}{Var}
\DeclareMathOperator{\Bias}{Bias}
\DeclareMathOperator*{\argmax}{argmax}
\DeclareMathOperator*{\argmin}{argmin}

\newtheorem{theorem}{Theorem}

\newtheorem{lemma}{Lemma}
\newtheorem{remark}{Remark}
\newtheorem{corollary}{Corollary}

\newtheorem{definition}{Definition}

\newcommand{\sR}{\mathbb{R}}

\newcommand{\EE}{\mathbb{E}}

\newcommand{\ind}[1]{\mathbbm{1}\{#1\}}
\newcommand{\tzero}{t_0}
\newcommand{\tsum}{t_{\Sigma}}
\newcommand{\pos}[1]{\left[#1\right]_+}

\newcommand{\dif}{\mathrm{d}}
\newcommand{\set}[1]{\left\lbrace #1 \right\rbrace}
\newcommand{\Bayes}{\varepsilon_{\mathrm{Bayes}}}

\newcommand{\parens}[1]{\left( #1 \right)}

\title{BOLT: Universal Classifier Training via a Bayes-Error Upper Bound}

\author{
  Mohammadreza Tavasoli Naeini,~\IEEEmembership{Member, IEEE,}
  Ali Bereyhi,~\IEEEmembership{Senior Member, IEEE,}
  Morteza Noshad,~\IEEEmembership{Member, IEEE,}
  Ben Liang,~\IEEEmembership{Fellow, IEEE,}
  and Alfred O. Hero,~\IEEEmembership{Life Fellow, IEEE}
 \thanks{This work was presented in part at \rev{the} 2025 IEEE International Conference on Acoustics, Speech, and Signal Processing (ICASSP) \cite{tavasoli2025universal}.}
\thanks{M.~Tavasoli Naeini, A.~Bereyhi, and B.~Liang are with the Department of Electrical and Computer Engineering at the University of Toronto, Toronto, ON, Canada (e-mail: mohammadreza.tavasolinaeini@mail.utoronto.ca; ali.bereyhi@utoronto.ca; liang@ece.utoronto.ca).}
\thanks{A.~O.~Hero is with the Department of Electrical Engineering and Computer Science at the University of Michigan, Ann Arbor, MI, USA (e-mail: hero@eecs.umich.edu).}
\thanks{M.~Noshad is with Stanford University, Stanford, CA, USA (e-mail: noshad@stanford.edu).}
\thanks{Portions of this work were completed while M.~Tavasoli Naeini was with the Department of EECS at the University of Michigan.}
}

\begin{document}

\begin{acronym}
\acro{map}[MAP]{maximum-a-posteriori-probability}
\acro{ber}[BER]{Bayes error rate}
\acro{bolt}[BOLT]{Bayes optimal learning threshold}
\acro{gan}[GAN]{generative adversarial network}
\acro{wgan}[WGAN]{Wasserstein GAN}
\acro{kl}[KL]{Kullback-Leibler}
\acro{rbm}[RBM]{restricted Boltzmann machine}
\acro{vae}[VAE]{variational autoencoder}
\acro{mmd}[MMD]{maximum mean discrepancy}
\acro{simclr}[SimCLR]{simple framework for contrastive learning of visual representations}
\acro{relu}[ReLU]{rectified linear unit}
\end{acronym}

\maketitle

\begin{abstract}
This work introduces a novel approach to compute a universal upper bound on the \ac{ber} from classifier-output samples\rev{,} \rev{which} \rev{serves} \rev{as} a Bayes-error-related performance criterion for the underlying classification task. We show that the bound is tight in the binary case and characterize the max-sum gap that appears in the aggregate multiclass formulation. Based on this bound, we define a new loss function, called Bayes optimal learning threshold (BOLT), for training parameterized classifiers. For each data sample, the BOLT loss computes a \rev{term} \rev{whose} \rev{expectation} upper\rev{-}\rev{bounds} the Bayes error\rev{,} \rev{so} \rev{that} \rev{during} training the model parameters \rev{can} \rev{be} updated toward minimizing this bound. Numerical experiments with both synthetic and real-world datasets show that, in the reported settings, BOLT training often attains lower held-out classification error than cross-entropy training.
\end{abstract}

\section{Introduction}
\label{sec:introduction}

A fundamental concept in machine learning is the inherent limit of classifier accuracy dictated by the \acl{ber} (\acs{ber}) \cite{duda2000patternclassification,devroye1996probabilistic}.
For a given classification problem, the \ac{ber}
is defined as the \rev{minimum} probability of \rev{misclassification} that any classifier can achieve \cite{devroye1996probabilistic}. The classifier that achieves this error rate is \emph{Bayes optimal}, as it minimizes the Bayes risk of the classifier with respect to the $0$-$1$ loss.
\rev{An} \rev{example} is the \ac{map} classifier\rev{,} \rev{which} assigns each data point to the class having the largest posterior probability \cite{duda2000patternclassification,bishop2006prml}. Understanding the \ac{ber} is crucial, as it provides a benchmark for comparing the performance of classifiers.

When the true underlying data distribution is known, the \ac{ber} \rev{has} \rev{an} \rev{explicit} \rev{integral} \rev{representation}. Consider the case of $m$-class classification, where the unknown class of a feature instance $x \in \sR^d$ is
$C \in \set{C_1, C_2, \ldots, C_m}$. Assume that the class-conditional data densities $p_i$ for $i\in \{1, \ldots, m\}$ are known, as are the prior class probabilities $\pi_i = \Pr\set{C = C_i}$. In this case, the \ac{ber} $\varepsilon_{\mathrm{Bayes}}$ is computed as \cite{bishop2006prml}
\begin{equation}
\label{BayesError}
\varepsilon_{\mathrm{Bayes}}
= 1 - \int_{\mathcal{X}} \max_{ i \in [m]} \left\{ \pi_i \, p_i\rev{(}\rev{x}\rev{)} \right\} \, \mathrm{d}x.
\end{equation}
The expression in \eqref{BayesError} is computationally hard in high-dimensional data spaces, even when the class-conditional densities are known. \rev{Moreover}\rev{,} \rev{in} most machine learning applications, the class-conditional densities are unknown. This has motivated the development of approximations and bounds on the \ac{ber}. This work develops a novel upper bound on the \ac{ber} that can be empirically estimated from training data when the data distribution is unknown.

\subsection{Contributions}
We establish a novel upper bound on the \ac{ber}, which can be empirically approximated by sampling from a learning model. The bound is \emph{universal} in the sense that it can be computed from output samples of an \emph{arbitrary} learning model, e.g., a neural network. We show that the bound is tight in the binary case. For the multiclass case, we characterize the max-sum gap between the aggregate BOLT bound and the \ac{ber} and derive explicit upper bounds on its size.
Motivated by these results, we propose a novel loss function that can be used to train a parameterized classifier toward minimizing the Bayes-error upper bound. We call this loss \emph{\ac{bolt}}. Using \ac{bolt}, the model training objective is aligned with \rev{minimizing} \rev{a} Bayes-\rev{error} \rev{upper} \rev{bound}, \rev{with} \rev{binary} \rev{tightness} \rev{and} \rev{a} \rev{bounded} \rev{multiclass} \rev{gap}\rev{.} \rev{Our} \rev{experiments} \rev{compare} \rev{BOLT} \rev{with} conventional cross-entropy \rev{training}.
The key contributions of this paper are four-fold.
\begin{enumerate}
    \item
    We establish a universal upper bound on the \ac{ber} for a general multiclass classification problem. The bound is derived by relating the Bayes error to an \( f \)-divergence between class-conditional distributions, evaluated by a hinge function.
    \item
    Invoking the universal bound, we propose a framework for \rev{estimating} an upper bound on the \ac{ber} from output samples of an arbitrary learning model. This \ac{ber} estimator defines a risk, i.e., expected loss, that can be estimated from data samples. We call the corresponding loss function \ac{bolt} and use it as a supervised surrogate to minimize \rev{a} \rev{bound} \rev{on} the \ac{ber}.
    \item We characterize the tightness of the proposed bound and the statistical accuracy of its sample estimator by deriving bias and variance bounds for \rev{a} \rev{likelihood}\rev{-}\rev{ratio} \rev{plug}\rev{-}\rev{in} \rev{estimator} \rev{of} the \rev{optimized} BOLT \rev{bound}. In the binary case, the max-sum gap vanishes, and the estimator is asymptotically unbiased as the likelihood-ratio estimation error vanishes and the sample size grows. In the multiclass case, the bias additionally contains the max-sum gap, whose size is explicitly bounded in terms of the number of classes and the \ac{ber}.
    \item
    We \rev{illustrate} the universal upper bound and \rev{evaluate} \rev{BOLT} \rev{training} through numerical experiments.
    We show that, on the reported held-out test splits, BOLT often attains lower classification error than the standard cross-entropy baseline. We \rev{also} \rev{compare} \rev{the} \rev{learned} \rev{representations} \rev{through} \rev{qualitative} feature \rev{projections}.
\end{enumerate}

\subsection{Organization}
\label{subsec:notation}
The remainder of this paper is organized as follows: In Section \ref{sec:related_work}, we give a brief overview of related work. Some background on Bayes error and \( f \)-divergence is provided in Section \ref{sec:background}. Sections \ref{sec:binary_bound} and \ref{sec:multiclass_bound} present the universal bound on the \ac{ber}
for binary and multiclass classification tasks, respectively. Section \ref{sec:BOLT} introduces the \acs{bolt} loss as a sample estimator of the error upper bound. Section \ref{sec:numerical_experiments} validates the proposed methods through experiments. Finally, Section \ref{sec:conclusion} concludes the paper and discusses potential directions for future work.

\subsection{Notation}
We use the shortened notations $[n]=\{1,\ldots,n\}$. A probability distribution is denoted by a capital letter, e.g., $P$, and its density, when it exists, is denoted by the corresponding lowercase letter, e.g., $p$. The notation $\mathbb{E}_{x\sim P}[\cdot]$ denotes expectation with respect to distribution $P$. When the distribution is known from the context, we write it as $\mathbb{E}_{x}[\cdot]$. We denote the $d\times d$ identity matrix by $I_d$ and a Gaussian distribution with mean $\mu$ and covariance matrix $\Sigma$ by $\mathcal{N}(\mu,\Sigma)$. The big-O notation is denoted as $\mathcal{O}(\cdot)$. For an estimator $T$ of an unknown parameter $\rev{b}$, we define the \rev{absolute} bias as
\begin{align}
    \Bias(T):=\big|\mathbb{E}[T]-\rev{b} \big|.
\end{align}
The variance of $T$ is denoted by $\Var(T)$. Other notations are defined in the text. 

\section{Related Work}
\label{sec:related_work}
Characterizing the \ac{ber} is crucial for assessing the theoretical limits of classifier performance in machine learning
\cite{cover1991elements,bishop2006prml,hastie2009elements}. Direct computation of the \ac{ber} is intractable in most applications because of the unknown data distribution and computational complexity, especially in high-dimensional spaces. This has led to several approaches for bounding or approximating the \ac{ber} \rev{across} disciplines ranging from pattern recognition \cite{devroye1996probabilistic,fukunaga1990statistical}, to information theory \cite{csiszar1967,cover1991elements}, and statistics \cite{liese2006}.
\subsection{Estimating Bayes Error}
Estimating and bounding the \ac{ber} has been a longstanding challenge in statistical pattern classification. Foundational work includes Fano's inequality \cite{fano1961transmission,cover1991elements}, Chernoff \rev{bounds} \cite{chernoff1952measure}, \rev{and} Kovalevskij's entropy-error bound \cite{kovalevskij1965problem}\rev{.} \rev{Hoeffding} \rev{concentration} \rev{bounds} \rev{\cite{hoeffding1963probability}} and Vapnik-\rev{-}Chervonenkis theory \cite{vapnik1971uniform} \rev{address} \rev{sampling} \rev{deviations} \rev{and} \rev{generalization}. Classical pattern-recognition bounds based on Mahalanobis and Bhattacharyya distances have also been widely used, as they capture class separability \cite{mahalanobis1936,bhattacharyya1943}. In high-dimensional spaces, these methods are affected by the curse of dimensionality, where the complexity of estimating distances or densities grows rapidly with feature dimension \cite{devroye1996probabilistic,fukunaga1990statistical}. Entropic quantities, including cross-entropy and \ac{kl} divergence, can characterize aspects of class separation, but they do not directly coincide with the \ac{ber} \cite{cover1991elements,bishop2006prml}.

Empirical methods for estimating Bayes error have shown substantial progress in recent years. Berisha et al.~\cite{berisha2016empirically} introduced empirically estimable classification bounds based on a nonparametric divergence measure, and Sekeh et al.~\cite{sekeh2020learning} developed multiclass \ac{ber} bounds based on generalized Henze-Penrose divergence. Ishida et al.~\cite{ishida2023} proposed a direct model-free estimator for binary classification that averages class-assignment uncertainties to produce unbiased and consistent estimates. The $\varepsilon$-ball estimator framework by Noshad et al.~\cite{noshad2019learning} ensures consistent Bayes-error estimates using local density estimators and can be applied to optimize classifiers. More recently, Ushio et al.~\cite{ushio2025practical} studied estimation of the optimal classification error using soft labels and calibration. Although these estimators are statistically attractive, neighborhood-based or direct nonparametric procedures can become expensive as the sample size and feature dimension grow.

Most existing methods rely on specialized estimators, explicit modeling assumptions, or distinct ensembling strategies to approximate or bound the \ac{ber}. In contrast, this work introduces an alternative method that uses a classifier model only through its output samples. This enables estimation of \rev{a} \rev{model}\rev{-}\rev{dependent} \rev{upper} \rev{bound} \rev{on} the \ac{ber} via the \ac{bolt} loss. The method is designed to align the training objective with the goal of reducing a \ac{ber} upper bound. This approach imposes the same per-epoch time complexity (FLOPs) as \rev{standard} \rev{empirical} risk minimization. This is in contrast to the neighborhood-based \ac{ber} estimator proposed by Noshad et al.~\cite{noshad2019learning}, which incurs substantially higher computational cost in high-dimensional spaces.

\subsection{Using Divergence Measures for Learning}
Divergence measures have been applied to quantifying class separability and have significantly influenced the development of robust machine learning models. The \ac{kl} divergence was an influential early \rev{measure}, laying the groundwork for subsequent divergence measures. The notion of $f$-divergence, introduced by Csisz{\'a}r in \cite{csiszar1967} and expanded by Liese and Vajda in \cite{liese2006}, provides a versatile \rev{measure} for assessing distributional differences.

Among various $f$-divergence-based error bounds, a line of work develops nonparametric divergence estimators that can be used to characterize the \ac{ber}. Moon \rev{and} \rev{Hero}~\cite{moon2014divergence} \rev{developed} \rev{ensemble} \rev{density} \rev{plug}\rev{-}in estimators of multivariate $f$-\rev{divergence}\rev{.} \rev{Sricharan} \rev{et} \rev{al}\rev{.}~\rev{\cite{sricharan2013divergence}} \rev{studied} \rev{weighted} \rev{ensembles} \rev{of} \rev{plug}\rev{-}\rev{in} \rev{estimators} \rev{for} \rev{entropy} and related information measures. \rev{Noshad} \rev{and} \rev{Hero}~\rev{\cite{noshad2017meta}} \rev{introduced} \rev{a} \rev{weighted} \rev{nearest}\rev{-}\rev{neighbor} \rev{graph} \rev{estimator} \rev{of} \rev{the} \rev{Henze}\rev{-}\rev{-}\rev{Penrose} \rev{divergence} \rev{for} \rev{bounding} binary classification error without explicit density estimation. More recent work has adopted these tools to benchmark modern neural networks against Bayes-optimal performance. Theisen et al.~\cite{theisen2021bayesopt} exploited normalizing flows and the invariance of the \ac{ber} under invertible transformations to approximate it on standard image benchmarks and to quantify how close state-of-the-art classifiers are to this limit. Chen et al.~\cite{chen2023bayeserror,chen2025efficientber} \rev{studied} Bayes-error \rev{estimation} \rev{through} \rev{identification} \rev{of} \rev{Bayes}-\rev{noise} \rev{samples}\rev{,} \rev{including} \rev{an} \rev{efficient} \rev{method} \rev{based} \rev{on} \rev{hierarchical} \rev{clustering} \rev{and} \rev{approximate} \rev{nearest} \rev{neighbors}. Wheat et al.~\cite{wheat2025difficult} compared $k$-\rev{nearest}\rev{-}\rev{neighbor}\rev{,} \rev{generalized} \rev{Henze}\rev{-}\rev{-}\rev{Penrose}\rev{,} \rev{and} \rev{kernel}\rev{-}\rev{density} \rev{estimators} \rev{in} synthetic \rev{binary}\rev{-}\rev{classification} \rev{settings}\rev{,} \rev{examining} \rev{their} \rev{accuracy} and \rev{uncertainty}.

The above-mentioned lines of work treat the \ac{ber} (or its bounds) primarily as a post-hoc assessment of trained models. In this study, we deviate from this convention by proposing a novel approach in which a divergence-based upper bound on the \ac{ber} is integrated into the training objective itself. This enables the training loop to \rev{reduce} \rev{an} \rev{upper} \rev{bound} \rev{on} the Bayes \rev{error}\rev{.}
\rev{Connections} \rev{between} \rev{classification} \rev{losses} \rev{and} $\rev{f}$-\rev{divergences} \rev{were} \rev{studied} \rev{by} \rev{Nguyen} \rev{et} \rev{al}\rev{.}~\rev{\cite{nguyen2009surrogate}} \rev{and} \rev{Reid} \rev{and} \rev{Williamson}~\rev{\cite{reid2011information}}\rev{.} \rev{Variational} \rev{estimation} \rev{of} \rev{divergences} \rev{and} \rev{likelihood} \rev{ratios} \rev{was} \rev{developed} \rev{by} \rev{Nguyen} \rev{et} \rev{al}\rev{.}~\rev{\cite{nguyen2010estimating}}.

\section{Universal Bound on Bayes Error}
\label{sec:method}
Consider an $m$-class classification problem on $\mathcal{X}\subseteq\mathbb{R}^d$, with class label $C\in\{C_1,\ldots,C_m\}$. Let $\rev{\Lambda}\rev{\in}\rev{[}\rev{m}\rev{]}$ \rev{be} \rev{the} \rev{class} \rev{index} \rev{and} \rev{denote} \rev{by} $\rev{\lambda}$ \rev{its} \rev{realization}\rev{.} \rev{Let} $\pi\rev{_\lambda}:=\Pr\{C=C\rev{_\lambda}\}\rev{>}\rev{0}$, and let $P\rev{_\lambda}$ and $p\rev{_\lambda}$ denote, respectively, the conditional distribution and density of $X$ given $C=C\rev{_\lambda}$, for $\rev{\lambda}\in[m]$.
The Bayes optimal classifier is given by the \ac{map} estimator, which estimates the class of sample $x\in\sR^d$ as
\begin{align}
\hat{C}(x) = \rev{C}_{\argmax_{\rev{\lambda}\rev{\in}\rev{[}\rev{m}\rev{]}} \Pr \{C = C\rev{_\lambda} \vert x\}}.
\end{align}
The posterior $\Pr \set{C = C\rev{_\lambda} \vert x}$ is computed from the prior probabilities $\pi\rev{_\lambda}$ and likelihoods $p\rev{_\lambda}(x)$ via Bayes' rule. The \ac{ber} is given by the expression in \eqref{BayesError}. In the sequel, we present some key notions and results that are used throughout the paper.

\subsection{Preliminaries}
\label{sec:background}
Following \cite{noshad2019learning}, the multiclass \ac{ber} admits an exact decomposition that is useful for our development. For $\lambda \in [m]$, define
\begin{equation}
\label{eq:mu_ver1}
\rev{r}_\lambda(x):=\max_{i\in[\lambda]} \pi_i p_i(x).
\end{equation}
Then the \ac{ber} can be written as
\begin{equation}
\label{Decompose}
\varepsilon_{\mathrm{Bayes}}
=1-\pi_1-\sum_{\lambda=2}^{m}\int_{\mathcal{X}}
\bigl[\rev{r}_\lambda(x)-\rev{r}_{\lambda-1}(x)\bigr]\,dx.
\end{equation}
This representation decomposes the multiclass \ac{ber} into incremental terms obtained by adding one class at a time, and it serves as the starting point for the divergence-based formulation developed in this section. The decomposition in \eqref{Decompose} is connected to $f$-divergence by observing that each incremental term admits a hinge-based divergence representation. Let us first recall the definition of $f$-divergence.
\begin{definition}[$f$-Divergence \cite{csiszar1967}]\label{def:f-Div}
Let distributions \( P \) and \( Q \) be defined on a measurable space \( \mathcal{X} \) with the probability density functions \( p(x) \) and \( q(x) \), respectively. Let \( f: (0, \infty) \rightarrow \mathbb{R} \) be convex. Then, the \( f \)-divergence between \( P \) and \( Q \) is defined as
\begin{equation}
D_f(P \parallel Q) = \int_{\mathcal{X}} q(x) f\left( \frac{p(x)}{q(x)} \right) \, \mathrm{d}x,
\end{equation}
where $f$ is referred to as the generator function. \rev{At} $\rev{p}\rev{=}\rev{0}\rev{<}\rev{q}$\rev{,} \rev{we} \rev{use} $\rev{f}\rev{(}\rev{0}\rev{)}\rev{=}\rev{\lim}_{\rev{u}\rev{\downarrow}\rev{0}}\rev{f}\rev{(}\rev{u}\rev{)}$\rev{;} \rev{at} $\rev{q}\rev{=}\rev{0}\rev{<}\rev{p}$\rev{,} \rev{the} \rev{integrand} \rev{is} $\rev{p} \rev{f}\rev{_\infty}$\rev{,} \rev{where} $\rev{f}\rev{_\infty}\rev{=}\rev{\lim}_{\rev{u}\rev{\to}\rev{\infty}}\rev{f}\rev{(}\rev{u}\rev{)}\rev{/}\rev{u}$\rev{;} \rev{and} \rev{the} \rev{integrand} \rev{is} \rev{zero} \rev{if} $\rev{p}\rev{=}\rev{q}\rev{=}\rev{0}$\rev{.} \rev{The} \rev{usual} \rev{normalization} \rev{is} $\rev{f}\rev{(}\rev{1}\rev{)}\rev{=}\rev{0}$\rev{,} \rev{while} \rev{unnormalized} \rev{generators} \rev{define} \rev{generalized} \rev{divergence} \rev{functionals}\rev{.}
\end{definition}

The $f$-divergence can be bounded in terms of population expectations. The relevant variational bound is stated in the following lemma.
\begin{lemma}[Variational lower bound on $f$-divergence {\cite{nowozin2016f}}]
\label{lem:variational_lower_bound}
Let $f^*:\mathbb{R}\to\mathbb{R}\rev{\cup}\{\rev{+}\rev{\infty}\}$ denote the Fenchel conjugate of the generator function $f$, defined as \cite{rockafellar1970convex}
\begin{equation}
f^*(\rev{s})=\sup_{u\in(0,\infty)} \{\rev{su}-f(u)\}.
\end{equation}
Then, the $f$-divergence between $P$ and $Q$ satisfies
\begin{equation}
\label{eq:variational_bound}
D_f(P\|Q)\geq
\sup_{h\in\mathcal{H}}
\left\{
\mathbb{E}_{X\sim P}[h(X)]
-
\mathbb{E}_{X\sim Q}[f^*(h(X))]
\right\},
\end{equation}
where $\mathcal{H}$ is a class of integrable functions $h:\mathcal{X}\to\mathbb{R}$ \rev{satisfying} $\rev{h}\rev{\le} \rev{f}\rev{_\infty}$ $\rev{P}$\rev{-}\rev{almost} \rev{everywhere} \rev{on} $\{\rev{q}\rev{=}\rev{0}\}$\rev{.} \rev{For} \rev{the} \rev{hinge} \rev{generators}\rev{,} \rev{this} \rev{condition} \rev{is} $\rev{h}\rev{\le}\rev{0}$ \rev{there}.
\end{lemma}

Using the variational bound above in \eqref{eq:variational_bound}, we can derive an upper bound on the \ac{ber} for a classification task. The derivation follows two main steps: ($i$) we relate the \ac{ber} to the $f$-divergence, and ($ii$) we use the variational bound to bound the \ac{ber}. For tractability, we first derive this bound for binary classification and then extend it to the general case.

\subsection{Bounding Binary BER}
\label{sec:binary_bound}
The first step of the derivation \rev{rests} \rev{on} a \rev{key} observation: the \ac{ber} can be expressed in terms of the \( f \)-divergence when $f$ is set to be the hinge loss. To see this point, let us first define the hinge loss.

\begin{definition}[Hinge Loss]
For \( u \in \mathbb{R} \), the hinge loss function is defined as
\begin{equation}
f_{\mathrm{hng}}(u) = \max\{0, 1 - u\}.
\end{equation}
\end{definition}

The hinge loss is known for its role as a surrogate loss for classification, e.g., in support vector machines. Using this function, we can directly represent the \ac{ber} as \rev{an} \rev{affine} \rev{function} \rev{of} an $f$-divergence with hinge generator.
\begin{lemma}[Binary Bayes error with uniform priors]
\label{thm:bayes_error_f_divergence}
In binary classification with equal prior probabilities $\pi_1=\pi_2=1/2$, the \ac{ber} satisfies
\begin{equation}
\varepsilon_{\mathrm{Bayes}}
=
\frac{1}{2}
-
\frac{1}{2}
D_{f_{\mathrm{hng}}}\!\left(P_1\|P_2\right).
\end{equation}
\end{lemma}

\begin{proof}
See Appendix~\ref{appendix:proof_uniform_prior_fdiv}.
\end{proof}

Using Lemma~\ref{thm:bayes_error_f_divergence}, we can bound the binary \ac{ber} directly via the variational bound on the $f$-divergence. This establishes a bound that can be estimated by sampling a learning model, i.e., the function $h$ in Lemma~\ref{lem:variational_lower_bound}.
This observation
can be extended to a general binary classification with arbitrary priors through the \emph{shifted} hinge loss and its Fenchel conjugate. To this end, let us first define the shifted hinge loss.

\begin{definition}[Shifted Hinge Loss]
For \( u \in \mathbb{R} \) and shift \( t > 0 \), the shifted hinge loss \( f_{\mathrm{shng}} \) is defined as
\begin{equation}
f_{\mathrm{shng}}(u; t) = \max\{0, t - u\}.
\end{equation}
\end{definition}
The Fenchel conjugate of the shifted hinge loss is linear \rev{on} a bounded interval, as shown in the following lemma.

\begin{lemma}[Fenchel conjugate of shifted hinge] \label{hinge_fenchel} The Fenchel conjugate of the shifted hinge loss \rev{restricted} \rev{to} $\rev{u}\rev{>}\rev{0}$ is given by \begin{equation} f_{\mathrm{shng}}^{*}(s) = \begin{cases} \rev{-}\rev{t}\rev{,} & \rev{\text{if }} \rev{s}\rev{<}\rev{-}\rev{1}\rev{,} \\ s t, & \text{if } -1 \leq s \leq 0, \\ +\infty, & \text{\rev{if} } \rev{s}\rev{>}\rev{0}. \end{cases} \end{equation}
\end{lemma}

\begin{proof}
See Appendix~\ref{appendix:proof_shifted_hinge_fenchel}.
\end{proof}

In a binary classification with an arbitrary prior, the \ac{ber} \rev{is} \rev{an} \rev{affine} \rev{function} \rev{of} the $f$-divergence with the shifted hinge loss as generator. This enables us to directly bound the error via the variational bound as given in the following theorem.
\begin{theorem}[Binary BER Bound]
\label{thm:binary_upper_bound}
Consider a binary classification task whose classes \( C_1 \) and \( C_2 \) occur with prior probabilities \( \pi_1 \) and \( \pi_2 \), respectively. For any integrable function \( h : \mathcal{X} \rightarrow [ -1, 0 ] \), the \ac{ber} \( \varepsilon_{\mathrm{Bayes}} \) satisfies
   \begin{align}
\varepsilon_{\mathrm{Bayes}} \leq &\;
\pi_2
+ \pi_2 \mathbb{E}_{X \sim P_2} [ h(X) ] - \pi_1 \mathbb{E}_{X \sim P_1} [ h(X) ].
\end{align}
\end{theorem}

\begin{proof}
Starting with the expression for the \ac{ber} in \eqref{BayesError}, we have in the binary case
\begin{subequations}
   \begin{align}
\varepsilon_{\mathrm{Bayes}} &= 1 - \int_{\mathcal{X}} \max \left\{ \pi_1 p_1(x), \pi_2 p_2(x) \right\} \mathrm{d}x, \\
&= 1 - \int_{\mathcal{X}} \left(\pi_1 \rev{p}\rev{_1}\rev{(}\rev{x}\rev{)}\rev{+}\rev{[}\rev{\pi}\rev{_2}p_2(x)\rev{-}\pi_1p_1(x)\rev{]}\rev{_+}\right) \mathrm{d}x.
\end{align}
\end{subequations}
Using the identity $\max\{ a,b \} = \rev{a} + \max\{0 , \rev{b}\rev{-}a \}$ and the fact that $\pi_2=1-\pi_1$, we can write
\begin{subequations}\label{eq:bayes_error}
   \begin{align}
\varepsilon_{\mathrm{Bayes}}
  &= 1 - \pi_1 \int_{\mathcal{X}} p\rev{_1}\rev{(}\rev{x}\rev{)}\,\rev{\mathrm{d}}\rev{x}
  \rev{-}\rev{\int}_{\rev{\mathcal{X}}}\rev{[}\rev{\pi}\rev{_2}\rev{p}\rev{_2}(x)\rev{-}\rev{\pi}\rev{_1}p_1(x)\rev{]}_+\,\mathrm{d}x,\\
  &= \pi_2 -
  {} \pi_1 \int_{\{\rev{x}\rev{:}\rev{p}\rev{_2}\rev{(}\rev{x}\rev{)}\rev{>}\rev{0}\}}
      \max\Biggl\{0,\,\frac{\pi_2}{\pi_1}
      - \frac{p_1(x)}{p_2(x)}\Biggr\}
      p_2(x) \, \mathrm{d}x.
\end{align}
\end{subequations}
\normalsize
\rev{The} \rev{positive}\rev{-}\rev{part} \rev{integrand} \rev{is} \rev{zero} \rev{wherever} $\rev{p}\rev{_2}\rev{(}\rev{x}\rev{)}\rev{=}\rev{0}$\rev{,} \rev{which} \rev{justifies} \rev{the} \rev{restriction} \rev{in} \rev{the} \rev{last} \rev{integral}\rev{.} The integrand is the shifted hinge loss applied to the likelihood ratio
\(
{ p_1(x) }/{ p_2(x) }
\)
whose shift is at \( { \pi_2 }/{ \pi_1 } \). As a result, the integral determines the $f$-divergence with shifted hinge generator. We can hence write
   \begin{align}
\varepsilon_{\mathrm{Bayes}}
&= \pi_2 - \pi_1 D_t ( P_1 \| P_2 ),\label{eq:bayes_fDiv}
\end{align}
where we define $D_t$ to be the $f$-divergence with shifted hinge generator at $ t = { \pi_2 }/{ \pi_1 } $, i.e.,
\[
D_{ t }( P_1 \| P_2 ) = \int_{\{\rev{x}\rev{:}\rev{p}\rev{_2}\rev{(}\rev{x}\rev{)}\rev{>}\rev{0}\}} f_{\mathrm{shng}}\left( \frac{ p_1(x) }{ p_2(x) }; t \right) p_2(x) \, \mathrm{d}x.
\]

\rev{For} $\rev{t}\rev{>}\rev{1}$\rev{,} $\rev{f}_{\rev{\mathrm{shng}}}\rev{(}\rev{1}\rev{;}\rev{t}\rev{)}\rev{=}\rev{t}\rev{-}\rev{1}$\rev{,} \rev{so} $\rev{D}\rev{_t}$ \rev{is} \rev{the} \rev{generalized} \rev{divergence} \rev{functional} \rev{defined} \rev{above}\rev{.} We now use the variational bound with the shifted hinge generator to bound $D_t(P\|Q)$. For any \rev{integrable} $h:\mathcal{X}\to\mathbb{R}$ \rev{satisfying} $\rev{h}\rev{\le}\rev{0}$ \rev{on} $\{\rev{q}\rev{=}\rev{0}\}$, it gives
\begin{align}
D_t( P \| Q ) \geq
\mathbb{E}_{X \sim P} [ h(X) ] - \mathbb{E}_{X \sim Q} [ f_{\mathrm{shng}}^{*}( h(X) ) ].
\end{align}
Lemma~\ref{hinge_fenchel} implies that if $h\rev{>}0$ on a set of positive $Q$-measure, then the variational lower bound is trivial. \rev{Values} \rev{below} $\rev{-}\rev{1}$ \rev{can} \rev{be} \rev{clipped} \rev{to} $\rev{-}\rev{1}$, \rev{since} \rev{this} \rev{increases} $\rev{\mathbb E}\rev{_P}\rev{[}\rev{h}\rev{]}$ \rev{while} \rev{leaving} $\rev{f}_{\rev{\mathrm{shng}}}\rev{^*}(\rev{h})$ \rev{unchanged}. We hence restrict $h$ to $h: \mathcal{X} \to [-1,0]$, for which we can write $f_{\mathrm{shng}}^{*}( h(X) ) =  h(X) t$. This concludes that
\begin{align}
D_{ t }( P_1 \| P_2 ) \geq
\mathbb{E}_{X \sim P_1} [ h(X) ] - t \mathbb{E}_{X \sim P_2} [ h(X) ] .
\end{align}
Substituting this bound in \eqref{eq:bayes_fDiv} and recalling that $t = {\pi_2}/{\pi_1}$, the bound in Theorem~\ref{thm:binary_upper_bound} is concluded.
\end{proof}

Theorem~\ref{thm:binary_upper_bound} \rev{holds} for any measurable $h$ \rev{with} \rev{values} \rev{in} $\rev{[}\rev{-}\rev{1}\rev{,}\rev{0}\rev{]}$. A suitable subclass of such measurable functions can be approximated by a parameterized model, e.g., a neural network. This leads to a straightforward implication of the result, stated in the following corollary.
\begin{corollary}[Computational bound on binary BER]
\label{Corollary:1}
    Let \(\mathcal{H} = \left\lbrace h_{\theta}: \mathcal{X} \to [-1,0] : \theta\in\Theta \right\rbrace\)
    be a set of integrable functions parameterized by \(\theta \in \Theta\),
    where \(\Theta\) is the parameter space defining the family of functions. The \ac{ber} satisfies
  \begin{equation}
\begin{split}
\varepsilon_{\mathrm{Bayes}} \le {} & \rev{\inf}_{\theta\in\Theta} \Bigl\{ \pi_2 + \pi_2\,\mathbb{E}_{X \sim P_2} [ h_\theta (X) ] - \pi_1\,\mathbb{E}_{X \sim P_1} [ h_\theta (X) ] \Bigr\}.
\end{split}
\end{equation}
\end{corollary}
\begin{proof}
    The \ac{ber} is bounded as in Theorem~\ref{thm:binary_upper_bound} for all $h_\theta\in\mathcal{H}$. We can hence conclude that it is also bounded by the \rev{infimum} \rev{of} \rev{these} \rev{upper} \rev{bounds} \rev{over} $\mathcal{H}$.
\end{proof}

Corollary~\ref{Corollary:1} gives a model-based approach to estimate an upper bound on the \ac{ber}. We show \rev{below} \rev{in} \rev{Corollary}~\rev{\ref{cor:binary_specialization}} that this bound is tight when $\mathcal{H}$ contains all measurable functions \rev{taking} \rev{values} \rev{in} $\rev{[}\rev{-}\rev{1},\rev{0}\rev{]}$. \rev{Section}~\ref{subsec:validation_gaussian} \rev{provides} \rev{a} \rev{binary} \rev{Gaussian} \rev{illustration}\rev{.} \rev{A} \rev{parameterized} \rev{family} \rev{approaches} \rev{this} \rev{value} \rev{if} \rev{it} \rev{approximates} \rev{a} \rev{binary} \rev{risk} \rev{minimizer} \rev{in} $\rev{L}\rev{^1}\rev{(}\rev{P}\rev{_1}\rev{+}\rev{P}\rev{_2}\rev{)}$\rev{,} \rev{since} \rev{the} \rev{objective} \rev{is} \rev{linear} \rev{in} $\rev{h}$.

 \subsection{Extension to \rev{Multiclass} Classification}
\label{sec:multiclass_bound}

A multiclass classification task can be viewed as a sequence of \rev{nested} binary classifications. In particular, \rev{each} \rev{class} $\rev{C}\rev{_\lambda}$ \rev{is} \rev{compared} \rev{with} the \rev{preceding} \rev{classes} $\rev{C}\rev{_1},\rev{\ldots},\rev{C}_{\rev{\lambda}-\rev{1}}$. This is the basis for the extension of the bound in Theorem~\ref{thm:binary_upper_bound} to a classification with an arbitrary number of classes.

\begin{theorem}[\rev{Multiclass} BER bound]
\label{thm:multiclass_upper_bound}
Consider a classification task with $m$ classes $C_1,\ldots,C_m$ on the data space $\mathcal{X}\subseteq\mathbb{R}^d$. Let
\[
h=[h_1,\ldots,h_m],
\]
where $h_1(x)\equiv -1$ and $h_i:\mathcal{X}\to[-1,0]$ is integrable for $i\in\{2,\ldots,m\}$. Then, the \ac{ber} is bounded as
\begin{equation}
\label{bound:general}
\varepsilon_{\mathrm{Bayes}}
\le
1-\sum_{\lambda=1}^{m}\pi_\lambda \mathcal{E}_\lambda,
\end{equation}
where, for each $\lambda\in[m]$, we define
\begin{equation}
\mathcal{E}_\lambda
=
\mathbb{E}_{X\sim P_\lambda}
\left[
\sum_{i=\lambda+1}^{m} h_i(X)-h_\lambda(X)
\right].
\end{equation}
\end{theorem}

\begin{proof}[Sketch of the Proof]
The proof starts with the Bayes error expression in~\eqref{Decompose}
with $ \rev{r}_\lambda$ being defined in~\eqref{eq:mu_ver1}. Define the increment
\begin{align}
\Delta_\lambda(x) \coloneqq  \rev{r}_\lambda(x) -  \rev{r}_{\lambda-1}(x),
\end{align}
which captures the contribution of adding class $C_\lambda$ to the MAP rule. For each $\lambda \ge 2$, we \rev{bound} the $\lambda$-th increment \rev{using}
a binary decision between class $C_\lambda$ and the \rev{preceding} \rev{classes} $\{\rev{C}\rev{_1}\rev{,}\rev{\ldots}\rev{,}\rev{C}_{\rev{\lambda}\rev{-}\rev{1}}\}$\rev{.} \rev{Define} \rev{their} aggregate \rev{prior} \rev{as}
\begin{align}
  \rev{\rho}_{\lambda-1} \coloneqq \rev{\sum}_{i=1}^{\lambda-1}\rev{\pi}_i\rev{.}
\end{align}
\rev{Let} $\rev{Q}_{\lambda-1}$ \rev{denote} \rev{the} \rev{distribution} \rev{with} \rev{density}
\begin{align*}
\rev{q}_{\lambda-1}\rev{(}\rev{x}\rev{)}\rev{\coloneqq}\rev{\rho}_{\rev{\lambda}\rev{-}\rev{1}}^{\rev{-}\rev{1}}\rev{\sum}_{\rev{i}\rev{=}\rev{1}}^{\rev{\lambda}\rev{-}\rev{1}}\rev{\pi}\rev{_i} \rev{p}\rev{_i}\rev{(}\rev{x}\rev{)}\rev{.}
\end{align*} \rev{The} \rev{normalized} \rev{priors} \rev{of} \rev{this} \rev{binary} \rev{subproblem} \rev{are} $\rev{\pi}\rev{_\lambda}\rev{/}\rev{(}\rev{\pi}\rev{_\lambda}\rev{+}\rev{\rho}_{\rev{\lambda}\rev{-}\rev{1}}\rev{)}$ \rev{and} $\rev{\rho}_{\rev{\lambda}\rev{-}\rev{1}}\rev{/}\rev{(}\rev{\pi}\rev{_\lambda}\rev{+}\rev{\rho}_{\rev{\lambda}\rev{-}\rev{1}}\rev{)}$.
This defines a binary subproblem whose Bayes error $\tilde{\varepsilon}_\lambda$
controls the integral of $\Delta_\lambda(x)$. Applying the upper bound in Theorem~\ref{thm:binary_upper_bound} to this binary subproblem with function $h_{\lambda}:\mathcal{X}\to[-1,0]$ yields an upper bound on $\tilde{\varepsilon}_\lambda$ in terms of $P_\lambda$ and $\rev{Q}_{\lambda-1}$. Substituting these bounds into
\eqref{Decompose}
and regrouping terms by classes the bound in \eqref{bound:general} is concluded. The detailed derivation is given in Appendix~\ref{app:proof_multiclass_upper_bound}.
\end{proof}

\section{Bayes Optimal Learning Threshold}
\label{sec:BOLT}
As shown in Theorem~\ref{thm:multiclass_upper_bound}, an upper bound on the \ac{ber} can be estimated from data samples using a suitable measurable function $h$. This insight motivates a new framework for supervised learning in which a parameterized model is trained to minimize its estimated error bound. In this section, we formulate this framework through defining the BOLT loss. We then analyze the proposed framework through characterization of the BOLT population risk and its optimizer.

\subsection{BOLT Loss}
Consider a supervised learning problem with the class-labeled dataset
\begin{equation}
\mathbb{D}
=
\left\{
(x_i,\lambda_i)\in\mathcal{X}\times{\rev{[}m\rev{]}}: i\in[n]
\right\},
\end{equation}
whose samples are independently and identically distributed (i.i.d.) according to a joint distribution that is determined by the class priors $\{\pi_\lambda\}_{\lambda=1}^m$ and the class-conditional distributions $\{P_\lambda\}_{\lambda=1}^m$. For the underlying discriminative learning task, we directly parameterize a bounded model
\begin{equation}
h_\theta:\mathcal{X}\to[-1,0]^m,
\end{equation}
where the model returns a set of scores for the classes, i.e.,
\begin{equation}
h_\theta(x)
=
\bigl[h_{\theta,1}(x),\ldots,h_{\theta,m}(x)\bigr]^\top,
\end{equation}
with the convention
\begin{equation}
h_{\theta,1}(x)\equiv -1,
\end{equation}
and the remaining coordinates being learnable bounded scores. In practice, these coordinates can be obtained by appropriate activation in the parameterized model, e.g., the negative sigmoid activation applied to real-valued logits in the last layer of a neural network.

Using the parameterized model $h_\theta$, we can estimate an upper bound on the \ac{ber} from the samples in $\mathbb{D}$. \rev{We} \rev{propose} \rev{training} the model \rev{by} \rev{minimizing} \rev{this} \rev{estimator}\rev{.} \rev{Our} \rev{later} \rev{analysis} \rev{establishes} \rev{binary} \rev{tightness} \rev{and} \rev{bounds} \rev{the} \rev{multiclass} \rev{gap} \rev{between} \rev{the} \rev{optimized} \rev{BOLT} \rev{risk} \rev{and} the \ac{ber}. Considering this proposed framework, we now define the resulting sample-level estimator of the \ac{ber} bound as a function, which we refer \rev{to} as \rev{the} BOLT loss.

\begin{definition}[BOLT Loss]
\label{def:BOLT_loss}
For an input $x\in\mathcal{X}$ with \rev{class} \rev{index} $\lambda\in[m]$, let
\begin{equation}
h_\theta(x)=\bigl[h_{\theta,1}(x),\ldots,h_{\theta,m}(x)\bigr]^\top,
\end{equation}
where $h_{\theta,1}(x)\equiv -1$ and $h_{\theta,\rev{\lambda}}(x)\in[-1,0]$ for $\rev{\lambda}\in\{2,\ldots,m\}$. The BOLT loss between the model output and the label is defined as
\begin{equation}
\ell_{\text{\normalfont{BOLT}}}\bigl(h_\theta(x),\lambda\bigr)
=
1-\sum_{i=\lambda+1}^{m} h_{\theta,i}(x)+h_{\theta,\lambda}(x).
\end{equation}
\end{definition}

The BOLT population risk determines the universal upper bound on the \ac{ber}. Let us write this risk as
\begin{subequations}
\label{eq:bolt_population_risk}
\begin{align}
\mathcal{L}_\theta
&= \mathbb{E}_{\rev{X},\rev{\Lambda}}\!\left[\ell_{\mathrm{BOLT}}(\rev{h}\rev{_\theta}\rev{(}\rev{X}\rev{)},\rev{\Lambda})\right],\\
&= \sum_{\lambda=1}^m \pi_{\lambda}\,
\mathbb{E}_{\rev{X} \sim P_\lambda}\!\left[\ell_{\mathrm{BOLT}}(\rev{h}\rev{_\theta}\rev{(}\rev{X}\rev{)},\lambda)\right],
\end{align}
\end{subequations}
where the second equality follows from the joint distribution of $(\rev{X},\rev{\Lambda})$.

\begin{lemma}
\label{lem:bolt_lemma}
The BOLT population risk bounds the \ac{ber} from above, i.e., $\varepsilon_{\mathrm{Bayes}}\le \mathcal{L}_\theta$.
\end{lemma}

\begin{proof}
Substituting the \ac{bolt} loss into the risk gives
\begin{equation}
\mathcal{L}_\theta
=
1-\sum_{\lambda=1}^{m}\pi_\lambda
\mathbb{E}_{x\sim P_\lambda}
\left[
\sum_{i=\lambda+1}^{m} h_{\theta,i}(x)-h_{\theta,\lambda}(x)
\right].
\label{eq:bolt_population_risk_expanded}
\end{equation}
Using the definition of $\mathcal{E}_\lambda$ in Theorem~\ref{thm:multiclass_upper_bound}, this can be written as
\begin{equation}
\mathcal{L}_\theta
=
1-\sum_{\lambda=1}^{m}\pi_\lambda \mathcal{E}_\lambda,
\end{equation}
which coincides with the upper bound in Theorem~\ref{thm:multiclass_upper_bound}. The result is then concluded straightforwardly from that theorem.
\end{proof}

Lemma~\ref{lem:bolt_lemma} suggests that minimizing the BOLT population risk guides the model to compute scores that yield a tighter bound on the \ac{ber}. In practice, we can estimate this risk from the samples in $\mathbb{D}$ and use a stochastic gradient-based optimizer to update the model. \rev{Although} intuitive, the properties of the model trained with this approach remain an open question. In the sequel, we characterize the minimizer of the BOLT population risk and its gap to the \ac{ber}. Our results reveal two key points: (i) \rev{an} \rev{endpoint} \rev{BOLT} \rev{risk} \rev{minimizer} \rev{gives} \rev{a} Bayes-optimal classifier \rev{in} \rev{binary} \rev{classification} \rev{and} \rev{a} \rev{bounded} \rev{excess} \rev{error} \rev{in} \rev{multiclass} \rev{classification}, and (ii) the \rev{optimized} BOLT population risk \rev{has} \rev{an} \rev{explicitly} \rev{bounded} gap \rev{above} the \ac{ber}. We \rev{also} \rev{characterize} \rev{a} \rev{likelihood}\rev{-}\rev{ratio} \rev{plug}\rev{-}\rev{in} \rev{estimate} of the \rev{optimized} \rev{bound}.

\subsection{BOLT Population Risk Minimization}
We characterize the BOLT population risk in two steps. (i) We first derive an alternative expression for the \ac{ber}, similar to the upper bound in Theorem~\ref{thm:multiclass_upper_bound}, as a sum of hinge terms involving pairwise likelihood ratios. (ii) We then show that the BOLT population risk is minimized by a sum-based threshold minimizer. This enables us to explicitly characterize the resulting gap between the BOLT population risk and the \ac{ber}. For the sake of brevity, we drop the dependence of $h$ on the model parameter $\theta$ in the sequel, wherever it is clear from the context. We start our derivation with the following lemma.

\begin{lemma}[Bayes error as a hinge sum]
\label{lem:multiclass_bayes_hinge}
For each class $\lambda\in\{2,\ldots,m\}$, define the class likelihood ratios
\begin{equation}
U_{j/\lambda}(x):=\frac{p_j(x)}{p_\lambda(x)}\quad\rev{\text{on }}\{\rev{p}\rev{_\lambda}\rev{(}\rev{x}\rev{)}\rev{>}\rev{0}\}
\label{eq:pairwise_likelihood_ratio}
\end{equation}
for $1\le j<\lambda$, and the hinge expression for real-valued arguments $u_{1},\ldots,u_{\lambda-1}$
\begin{equation}
\tzero^{(\lambda)}(u_{1},\ldots,u_{\lambda-1})
:=
\pos{\pi_\lambda-\max_{1\le j<\lambda}\pi_j u_{j}},
\end{equation}
where $\pos{x}=\max\set{0,x}$. Then, the \ac{ber} can be written as
\begin{equation}
\Bayes
=
1-\pi_1-
\sum_{\lambda=2}^m
\EE_{x\sim P_\lambda}\!\Big[
\tzero^{(\lambda)}\!\big(
U_{1/\lambda}(x),\ldots,U_{(\lambda-1)/\lambda}(x)
\big)
\Big].
\end{equation}
\end{lemma}

\begin{proof}
See Appendix~\ref{app:proof_multiclass_bayes_hinge}.
\end{proof}

\begin{theorem}[BOLT minimizer and max-sum gap]
\label{thm:multiclass_optimizer_tightness}
Consider the BOLT population risk in \eqref{eq:bolt_population_risk_expanded}, with dependence on $\theta$ suppressed, where $h_1(x)\equiv -1$ and $h_i:\mathcal X\to[-1,0]$ for $i=2,\ldots,m$. For $\lambda\ge2$, define
\begin{subequations}
\begin{align}
S_{\lambda-1}(x)&\coloneqq \sum_{j=1}^{\lambda-1}\pi_j p_j(x),\\
 \rev{r}_{\lambda-1}(x)&\coloneqq \max_{1\le j<\lambda}\pi_j p_j(x).
\end{align}
\end{subequations}
Then, a pointwise minimizer is given by $h_1^\star(x)\equiv -1$ and
\begin{equation}
h^\star_\lambda(x)=
\begin{cases}
-1, & \pi_\lambda p_\lambda(x)>S_{\lambda-1}(x),\\[2pt]
0, & \pi_\lambda p_\lambda(x)\le S_{\lambda-1}(x),
\end{cases}
\end{equation}
for $\lambda=2,\ldots,m$. Furthermore, the minimal risk $\mathcal{L}^\star$ is given by
\begin{equation}
\mathcal{L}^\star
=
1-\pi_1-
\sum_{\lambda=2}^{m}
\int_{\mathcal X}
\pos{\pi_\lambda p_\lambda(x)-S_{\lambda-1}(x)}
\,\dif x,
\label{eq:lbolt_sum_optimizer_value}
\end{equation}
and its gap to the \ac{ber}, i.e., $\Delta_{\mathrm{ms}}\coloneqq\mathcal{L}^\star-\Bayes$, is
\begin{equation}
\Delta_{\mathrm{ms}}
=
\sum_{\lambda=2}^{m}
\int_{\mathcal X}
\min\left\{
\pos{\pi_\lambda p_\lambda(x)- \rev{r}_{\lambda-1}(x)},
S_{\lambda-1}(x)- \rev{r}_{\lambda-1}(x)
\right\}
\,\dif x.
\label{eq:lbolt_bayes_gap_delta}
\end{equation}
\end{theorem}

\begin{proof}
See Appendix~\ref{app:proof_multiclass_optimizer_tightness}.
\end{proof}

\rev{Theorem}~\rev{\ref{thm:multiclass_optimizer_tightness}} \rev{has} \rev{interesting} \rev{implications}\rev{.} The first key point is the fact that we are able to build a classifier directly from the learned score model $h$. For any endpoint-valued score model $h:\mathcal X\to\{-1,0\}^m$ with $h_1(x)\equiv-1$, define
\begin{equation}
C_h(x)
\coloneqq
C_{\max\{i\in[m]:h_i(x)=-1\}}.
\label{eq:classifier_from_bolt_scores}
\end{equation}
The maximization set is nonempty because $h_1(x)=-1$. For an endpoint-valued score model, the BOLT loss is zero for the class selected by \eqref{eq:classifier_from_bolt_scores} and is at least one for every other class. Hence,
\begin{equation}
\ind{C_h(x)\ne C_\lambda}
\le
\ell_{\mathrm{BOLT}}\bigl(h(x),\lambda\bigr).
\end{equation}
The second key point is that the classifier \rev{induced} \rev{by} the \rev{endpoint} \rev{population} \rev{minimizer} \rev{has} \rev{error} \rev{at} \rev{most} the \ac{ber} \rev{plus} \rev{the} \rev{max}\rev{-}\rev{sum} \rev{gap}. More precisely,
\begin{equation}
\Bayes
\le
\Pr\!\left\{C_{h^\star}(\rev{X})\ne C\right\}
\le
\mathcal{L}^\star
=
\Bayes+\Delta_{\mathrm{ms}}.
\label{eq:induced_classifier_risk}
\end{equation}
Thus, the gap is zero in the binary case and can be nonzero in the multiclass case. \rev{The} \rev{class} \rev{order} \rev{affects} \rev{the} \rev{bound}\rev{,} \rev{the} \rev{minimizer} \rev{and} \rev{the} \rev{induced} \rev{classifier}\rev{.}
\rev{For} \rev{continuous} \rev{bounded} \rev{scores}\rev{,} \rev{a} \rev{minimum}\rev{-}\rev{loss} \rev{rule} \rev{is} $\rev{C}_{\rev{\argmin}_{\rev{\lambda}\rev{\in}\rev{[}\rev{m}\rev{]}}\rev{\ell}_{\rev{\mathrm{BOLT}}}\rev{(}\rev{h}\rev{(}\rev{x}\rev{)}\rev{,}\rev{\lambda}\rev{)}}$\rev{,} \rev{with} \rev{ties} \rev{resolved} \rev{by} \rev{the} \rev{smallest} \rev{index}\rev{.} \rev{It} \rev{agrees} \rev{with} \rev{\eqref{eq:classifier_from_bolt_scores}} \rev{at} \rev{endpoint}\rev{-}\rev{valued} \rev{scores}\rev{,} \rev{to} \rev{which} \rev{the} \rev{inequalities} \rev{above} \rev{apply}\rev{.}

\begin{corollary}[Binary BOLT minimizer]
\label{cor:binary_specialization}
In the binary case, i.e., $m=2$, define $\rev{t}\coloneqq {\pi_2}/{\pi_1}$ and $U(x)\coloneqq{p_1(x)}/{p_2(x)}$ \rev{on} $\{\rev{p}\rev{_2}\rev{>}\rev{0}\}$. Then, the binary \ac{ber} is written as
\begin{equation}
\varepsilon_{\mathrm{Bayes}}
=
\pi_2-
\mathbb{E}_{x\sim P_2}\!\bigl[t_0(U(x))\bigr],
\label{eq:bayes_error_expectation}
\end{equation}
for the hinge function $t_0(u):=\big[\pi_2-\pi_1u\big]_+$. Moreover, let $h_1^\star(x)\equiv -1$, and
\begin{equation}
h_2^\star(x)=
\begin{cases}
-1, & \rev{p}\rev{_2}\rev{(}\rev{x}\rev{)}\rev{>}\rev{0}\rev{\text{ and }}U(x)<\rev{t},\\[2pt]
0, & \rev{\text{otherwise}}.
\end{cases}
\end{equation}
Then, the BOLT population risk is tight at $h^\star=(h_1^\star,h_2^\star)$, i.e., $\mathcal{L}^\star=\Bayes$, and the \ac{map} classifier is recovered from $h^\star$ as
\begin{equation}
C_{\mathrm{MAP}}(x)=
\begin{cases}
C_1, & \text{if } h_2^\star(x)=0,\\[2pt]
C_2, & \text{if } h_2^\star(x)=-1.
\end{cases}
\end{equation}
\end{corollary}

\begin{proof}
See Appendix~\ref{app:proof_binary_specialization}.
\end{proof}

\begin{remark}[Relation to the binary hinge loss]
Using the bipolar representation in the binary case, i.e., $s=+1$ for $\lambda=1$ and $s=-1$ for $\lambda=2$, and defining $g(x)\coloneqq 2h_2(x)+1$, the BOLT loss can be \rev{rewritten} in this case as\footnote{Because $h_2(x)\in[-1,0]$, we have $g(x)\in[-1,1]$ and therefore $1-sg(x)\ge0$. Thus, the positive-part operator in the standard hinge loss is redundant in this bounded-score parameterization.}
\begin{equation}
\ell_{\mathrm{BOLT}}\bigl(h(x),\lambda\bigr)
=
\frac{1}{2}\bigl(1-sg(x)\bigr).
\label{eq:binary_bolt_hinge_equivalence}
\end{equation}
This is the same unit-margin hinge term used in support vector machines (SVMs), scaled by $1/2$ and evaluated on a bounded score\rev{.} \rev{If} $\rev{h}\rev{_2}\rev{(}\rev{x}\rev{)}\rev{=}\rev{-}\rev{\sigma}\rev{(}\rev{z}\rev{(}\rev{x}\rev{)}\rev{)}$\rev{,} \rev{then} \rev{the} \rev{binary} \rev{loss} \rev{is} $\rev{\sigma}\rev{(}\rev{sz}\rev{)}$\rev{,} \rev{a} \rev{symmetric} \rev{sigmoid} \rev{loss} \rev{studied} \rev{in} \rev{\cite{ghosh2017robust,charoenphakdee2019symmetric}}.
\end{remark}

\begin{theorem}[Upper bound on the max-sum gap]
\label{thm:max_sum_gap_size}
Let
\begin{equation}
\Delta^\star
\coloneqq
\min\left\{
\parens{1-2^{-(m-2)}}\Bayes,
\frac{m(1-\Bayes)-1}{2}
\right\}.
\label{eq:max_sum_gap_upper_definition}
\end{equation}
The max-sum gap $\Delta_{\mathrm{ms}}$ is bounded from above as
\begin{equation}
0\le \Delta_{\mathrm{ms}}\le\Delta^\star\le\Bayes.
\label{eq:max_sum_gap_bayes_bound}
\end{equation}
Let
\begin{equation}
\Bayes^\star
\coloneqq
\frac{m-1}{m+2\parens{1-2^{-(m-2)}}}.
\label{eq:gap_bound_switch_point}
\end{equation}
The upper bound $\Delta^\star$ satisfies the following properties:
\begin{enumerate}
\item If $0\le\Bayes\le\Bayes^\star$, then
$\Delta^\star=\parens{1-2^{-(m-2)}}\Bayes$.
\item If $\Bayes^\star\le\Bayes\le(m-1)/m$, then
$\Delta^\star=\bigl[m(1-\Bayes)-1\bigr]/2$.
\end{enumerate}
\end{theorem}

\begin{proof}
See Appendix~\ref{app:proof_max_sum_gap_size}.
\end{proof}

The upper bound in Theorem~\ref{thm:max_sum_gap_size} is always less than or equal to the \ac{ber} and is small when the \ac{ber} is small, which is often the case in practice. Combining \eqref{eq:induced_classifier_risk} and \eqref{eq:max_sum_gap_bayes_bound}, the classifier built from the BOLT population risk minimizer can never have error larger than $2\Bayes$. The bound depends on the number of classes and the \ac{ber}. Its second branch also vanishes in the random-label regime with uniform posterior mass, where $\Bayes=(m-1)/m$.

\subsection{Learning via Empirical BOLT Risk Minimization}
\label{sec:risk_min_bolt}
In practice, the BOLT algorithm estimates the risk via a batch estimator computed on a mini-batch $\mathcal{B}\subset\mathbb{D}$. This means that at each iteration of a gradient-based optimization algorithm, the BOLT risk is estimated as
\begin{equation}
\hat{\mathcal{L}}_\theta
=
\frac{1}{|\mathcal{B}|}
\sum_{i\in\mathcal{B}}\ell_{\mathrm{BOLT}}(y_i,\lambda_i),
\end{equation}
with $y_i=h_\theta(x_i)$ being the sample model output. The risk gradient in this case is estimated as
\begin{equation}
\nabla\rev{_\theta}\hat{\mathcal{L}}_\theta
=
\frac{1}{|\mathcal{B}|}
\sum_{i\in\mathcal{B}}\nabla\rev{_\theta}\ell_{\mathrm{BOLT}}(y_i,\lambda_i),
\end{equation}
and used to update the model parameter $\theta$. Noting that BOLT is computed by an analytic expression, we can readily compute the gradient estimator at each iteration via standard backpropagation.

\rev{An} \rev{independent} \rev{test}\rev{-}\rev{sample} \rev{average} \rev{of} \rev{the} \rev{BOLT} \rev{loss} \rev{estimates} \rev{the} \rev{trained} \rev{model}\rev{'}\rev{s} \rev{population} \rev{risk} $\rev{\mathcal L}\rev{_\theta}$\rev{.} We next characterize \rev{a} \rev{separate} \rev{likelihood}\rev{-}\rev{ratio} \rev{plug}\rev{-}\rev{in} \rev{estimator} of the \rev{optimized} BOLT bound. In general, this estimator contains the deterministic max-sum gap identified in Theorem~\ref{thm:multiclass_optimizer_tightness}, as well as \rev{ratio}\rev{-}approximation and finite-sample errors. We characterize \rev{its} \rev{conditional} bias and variance. To this end, let
\begin{equation}
\widehat U_{\cdot/\lambda}(x)
:=
\bigl(\widehat U_{1/\lambda}(x),\ldots,\widehat U_{(\lambda-1)/\lambda}(x)\bigr),
\end{equation}
for $\lambda\in\{2,\ldots,m\}$, denote \rev{nonnegative} \rev{measurable} \rev{estimates}\rev{,} \rev{fitted} \rev{on} \rev{an} \rev{auxiliary} \rev{set} \rev{of} \rev{samples} $\rev{\mathcal T_U}$\rev{,} of the likelihood ratios in the vector
\begin{equation}
U_{\cdot/\lambda}(x)
:=
\bigl(U_{1/\lambda}(x),\ldots,U_{(\lambda-1)/\lambda}(x)\bigr),
\end{equation}
with $U_{j/\lambda}(x)$ defined in Lemma~\ref{lem:multiclass_bayes_hinge}. \rev{These} \rev{ratio} \rev{estimates} \rev{are} \rev{additional} \rev{inputs}\rev{,} \rev{since} \rev{bounded} \rev{BOLT} \rev{scores} \rev{alone} \rev{do} \rev{not} \rev{determine} \rev{them}\rev{.} Let $M_\lambda\rev{>}\rev{0}$ denote the \rev{fixed} number of samples from class $C_\lambda$, used to estimate the \rev{optimized} bound, and let
\begin{equation}
M:=\sum_{\lambda=1}^m M_\lambda
\end{equation}
be the total number of test samples. Define the aggregate hinge term as
\begin{equation}
\tsum^{(\lambda)}(u_{1/\lambda},\ldots,u_{(\lambda-1)/\lambda})
:=
\left[
\pi_\lambda-
\sum_{j=1}^{\lambda-1}\pi_j\,u_{j/\lambda}
\right]_+.
\end{equation}
The plug-in estimator of the \rev{optimized} \rev{BOLT} bound is then given by
\begin{equation}
\widehat{\varepsilon}_{\mathrm{BOLT}}
:=
1-\pi_1
-\sum_{\lambda=2}^m\frac{1}{M_\lambda}\sum_{i=1}^{M_\lambda}
\tsum^{(\lambda)}
\bigl(
\widehat U_{\cdot/\lambda}(x_{\lambda,i})
\bigr),
\label{eq:eps_hat_bolt_multiclass}
\end{equation}
where $x_{\lambda,i}\sim P_\lambda$. \rev{Its} \rev{bias} \rev{relative} \rev{to} the \ac{ber} \rev{includes} \rev{the} \rev{max}\rev{-}\rev{sum} \rev{gap}.

\begin{theorem}[Bias and variance of the \rev{optimized}\rev{-}\rev{bound} \rev{plug}\rev{-}\rev{in} \rev{estimator}]
\label{thm:multiclass_bias_variance}
Let $\Delta_{\mathrm{ms}}$ be the max-sum gap in Theorem~\ref{thm:multiclass_optimizer_tightness}. Assume that \rev{the} \rev{target} \rev{priors} $\rev{\pi}\rev{_\lambda}$ \rev{are} \rev{known} \rev{and} \rev{that} \rev{all} \rev{test} \rev{samples} \rev{are} \rev{independent} \rev{of} $\rev{\mathcal T_U}$ \rev{and} \rev{of} \rev{one} \rev{another}\rev{,} \rev{and} \rev{that}\rev{,} \rev{almost} \rev{surely} \rev{in} $\rev{\mathcal T_U}$\rev{,}
\begin{equation}
\mathbb{E}_{x\sim P_\lambda}\!\left[
\|\widehat U_{\cdot/\lambda}(x)-U_{\cdot/\lambda}(x)\|_\infty
\,\middle|\,\rev{\mathcal T_U}\right]
\le \varepsilon_0,
\end{equation}
for $\lambda\in\set{2,\ldots,m}$, and let
$M^\star\coloneqq\min\set{M_1,\ldots,M_m}$.
Then, the \rev{conditional} \rev{absolute} bias \rev{relative} \rev{to} the \ac{ber} \rev{is} \rev{bounded} \rev{as}
\begin{equation}
\left|\rev{\mathbb E}\rev{[}\widehat{\varepsilon}_{\mathrm{BOLT}}\rev{\mid}\rev{\mathcal T_U}\rev{]}\rev{-}\rev{\Bayes}\right|
\le
\Delta_{\mathrm{ms}}
+(m-1)\varepsilon_0,
\label{eq:multi_bias_bound}
\end{equation}
and its \rev{conditional} variance is
\begin{equation}
\Var\!\parens{\widehat{\varepsilon}_{\mathrm{BOLT}}\rev{\mid}\rev{\mathcal T_U}}
=
\mathcal{O}\!\parens{\frac{(m-1)^2}{M^\star}}.
\label{eq:multi_var_bound}
\end{equation}
\end{theorem}

\begin{proof}
See Appendix~\ref{app:proof_multiclass_bias_variance}.
\end{proof}

Theorem~\ref{thm:multiclass_bias_variance} indicates that the \rev{plug}\rev{-}\rev{in} \rev{estimator} \rev{is} \rev{asymptotically} \rev{exact} \rev{for} \rev{the} \ac{ber} when the BOLT population risk is tight\rev{,} the likelihood-ratio approximation error vanishes\rev{,} \rev{and} $\rev{M}\rev{^\star}$ \rev{grows}. Binary classification is one such setting because the max-sum gap is zero\rev{.} \rev{The} \rev{sum}\rev{-}\rev{based} \rev{hinge} \rev{is} \rev{used} \rev{here} \rev{to} \rev{characterize} \rev{the} \rev{optimized} \rev{BOLT} \rev{bound}.

\section{Numerical Experiments}
\label{sec:numerical_experiments}
Next, we \rev{illustrate} the \rev{binary} tightness \rev{result} and evaluate the BOLT framework using both synthetic and real-world data. To this end, we consider the following experimental settings:
\begin{enumerate}
    \item We first study a binary Gaussian classification problem where the BOLT population risk is tight and the Bayes-optimal classifier is derived analytically.
    \item We evaluate the BOLT framework on a synthetic multiclass classification dataset with samples whose Bayes-optimal decision regions are nonlinear and radial, and examine whether the framework can closely approximate the \ac{ber}.
    \item We next train multiple classifiers on realistic visual and textual datasets and compare the BOLT framework with standard cross-entropy minimization.
    \item Finally, we examine the learned embeddings and the behavior of BOLT-trained models under magnitude pruning.
\end{enumerate}

\subsection{Univariate Gaussian Data}
\label{subsec:validation_gaussian}
We consider a binary classification task in which the classes $C_1$ and $C_2$ are Gaussian distributed with different means and identical variances. Specifically, the class-conditional probability density functions for $i\in\{1,2\}$ are
\begin{align}
  p_i(x)=\frac{1}{\sqrt{2\pi\sigma^2}}\exp\set{-\frac{(x-\mu_i)^2}{2\sigma^2}}.
\end{align}
We assume a uniform class prior, i.e., $\pi_1=\pi_2=0.5$. For this problem, the \ac{ber} is given by \cite{bishop2006prml}
\begin{equation}
\varepsilon_{\mathrm{Bayes}}=Q\left(\frac{|\Delta\mu|}{2\sigma}\right),
\label{eq:Bayes_error_gaussian}
\end{equation}
where $\Delta\mu=\mu_1-\mu_2$ and $Q(x)$ denotes the Q-function,
\begin{align}
Q(x)=\frac{1}{\sqrt{2\pi}}\int_x^\infty \exp\set{-\frac{u^2}{2}}\dif u.
\end{align}
The Bayes-optimal classifier is the MAP classifier
\begin{align}
\widehat C(x)=
\begin{cases}
C_1, & \Delta\mu\,(x-\bar\mu)\ge 0,\\
C_2, & \Delta\mu\,(x-\bar\mu)<0,
\end{cases}
\end{align}
where $\bar\mu\coloneqq(\mu_1+\mu_2)/2$. This sign-aware form covers both positive and negative values of $\Delta\mu$ used in the experiment.

From Corollary~\ref{cor:binary_specialization}, we know that the universal bound in this case is tight. In the sequel, we \rev{verify} this result analytically and \rev{evaluate} \rev{BOLT} \rev{training} through a simple neural network.

\paragraph*{Validation of Tightness}
From Theorem~\ref{thm:binary_upper_bound}, the \ac{ber} satisfies
\begin{equation}
\varepsilon_{\mathrm{Bayes}}
\leq
\pi_2+\pi_2\mathbb{E}_{X\sim P_2}[h(X)]
-\pi_1\mathbb{E}_{X\sim P_1}[h(X)],
\end{equation}
where $h:\mathcal X\to[-1,0]$. This bound is minimized by
\begin{align}
h^\star(x)=\dot f_{\mathrm{shng}}\left(\frac{p_1(x)}{p_2(x)}\right),
\end{align}
where $\dot f_{\mathrm{shng}}(u)=-1$ for $u<1$ and $0$ otherwise. Here, $t=\pi_2/\pi_1=1$ because the class prior is uniform. The likelihood ratio is
\begin{align}
\frac{p_1(x)}{p_2(x)}
=
\exp\set{\frac{\Delta\mu}{\sigma^2}(x-\bar\mu)}.
\end{align}
Comparing this likelihood ratio with $t=1$ gives the threshold $\bar\mu$. Thus,
\begin{align}
h^\star(x)=
\begin{cases}
-1, & \Delta\mu\,(x-\bar\mu)<0,\\
0, & \Delta\mu\,(x-\bar\mu)\ge0.
\end{cases}
\end{align}
\begin{remark}
In addition to demonstrating the tightness of the bound, this result confirms that the optimal classifier can be constructed directly from the function $h^\star(x)$ that optimizes the bound.
\end{remark}

\paragraph*{Empirical Validation via Trained Models}
To \rev{evaluate} \rev{BOLT} \rev{training}, we train a simple neural network with the BOLT loss using samples from two unit-variance Gaussian distributions with means $\mu_1$ and $\mu_2$. The network has two fully connected hidden layers with 100 and 50 neurons, respectively, and rectified linear unit (ReLU) activations. Its output layer contains one sigmoid-activated neuron whose output is multiplied by $-1$ so that it lies in $(-1,0)$. We generate 20,000 training samples, with 10,000 samples from each class, and vary $\mu_1-\mu_2$ from $-6$ to $6$. We train the network using Adam with learning rate $\eta=0.001$ for 1000 iterations and a mini-batch size of 100, and evaluate it on a separate set of 1000 samples. Figure~\ref{fig:bayes_comparison} compares the held-out classification error of the BOLT-trained model with the analytic \ac{ber}. \rev{This} \rev{comparison} \rev{evaluates} \rev{classification} \rev{performance}\rev{;} the \rev{preceding} \rev{derivation} \rev{establishes} \rev{binary} \rev{tightness}\rev{.} \rev{With} \rev{1000} \rev{test} \rev{samples}\rev{,} \rev{sampling} \rev{variation} \rev{can} \rev{place} \rev{an} \rev{empirical} \rev{error} \rev{below} the \rev{population} \rev{Bayes} \rev{error}.
\begin{figure}[t!]
    \centering
\begin{tikzpicture}[font=\footnotesize]
  \begin{axis}[
    width=0.66\linewidth,
    height=0.45\linewidth,
    xlabel={$\mu_1-\mu_2$},
    ylabel={classification error},
    xmin=-6, xmax=6,
    ymin=0, ymax=0.55,
    tick label style={font=\scriptsize},
    label style={font=\footnotesize},
    legend style={font=\scriptsize},
    legend pos=north east,
    grid=both,
    grid style={opacity=.12},
    every mark/.append style={solid, mark size=1.8pt}
  ]
    \addplot[color=red, mark=*, thick] coordinates {
            (-6.0, 0.0013499) (-5.3333, 0.00383038) (-4.6667, 0.00981533) 
            (-4.0, 0.02275013) (-3.3333, 0.04779035) (-2.6667, 0.09121122)
            (-2.0, 0.15865525) (-1.3333, 0.25249254) (-0.6667, 0.36944134)
            (0.0, 0.5) (0.6667, 0.36944134) (1.3333, 0.25249254)
            (2.0, 0.15865525) (2.6667, 0.09121122) (3.3333, 0.04779035)
            (4.0, 0.02275013) (4.6667, 0.00981533) (5.3333, 0.00383038)
            (6.0, 0.0013499)
        
    };
    \addlegendentry{$\varepsilon_{\mathrm{Bayes}}$}

    \addplot[color=green!60!black, mark=square*, thick] coordinates {
            (-6.0,  0.002924799919128418) (-5.3333, 0.0033133119344711305) (-4.6667, 0.006892746686935425) 
            (-4.0, 0.025854086875915526) (-3.3333, 0.04903736710548401) (-2.6667,0.0860111892223358)
            (-2.0, 0.15536127686500557) (-1.3333, 0.2581663012504577 )(-0.6667,0.3700394004583359)
            (0.0, 0.509317138791084) (0.6667, 0.38276363015174864) (1.3333, 0.24733462929725647)
            (2.0, 0.15572800040245055) (2.6667, 0.08786912858486176) (3.3333, 0.050270494818687436)
            (4.0, 0.027295491099357604) (4.6667, 0.011860710382461549) (5.3333, 0.00601939857006073)
            (6.0, 0.0025591999292373657)
        
    };
    \addlegendentry{BOLT}
  \end{axis}
\end{tikzpicture}
    \caption{\rev{Analytic} BER \rev{and} \rev{BOLT} \rev{test} \rev{classification} \rev{error} versus class-mean separation\rev{,} with 20,000 training samples \rev{and} \rev{1000} \rev{test} \rev{samples}.}
    \label{fig:bayes_comparison}
\end{figure}

\subsection{Multiclass Concentric Gaussian Data}
\label{subsec:synthetic_setup}
We evaluate BOLT on a four-class classification problem with nonlinear radial decision regions. The data synthesis procedure is as follows. For $\lambda\in\{1,2,3,4\}$, samples are drawn from
\begin{equation}
X\mid \rev{(}\rev{C}\rev{=}C_\lambda\rev{)}\sim\mathcal N(0,\sigma_\lambda^2 I_d)
\end{equation}
in $d=20$ dimensions under a uniform class prior. We set $(\sigma_1,\sigma_2,\sigma_3,\sigma_4)=(1.0,1.5,2.0,2.5)$, use training-set sizes $n\in\{400,800,1600,3200,6400,12800,25600\}$, and evaluate on 5000 test samples. Because the covariance scale varies across classes, the Bayes-optimal classifier depends on the sample radius rather than coordinate-wise thresholds.

We compare \rev{four} \rev{classifiers}: random forest (RF) with 300 trees and Gini splitting, gradient boosting (GB) with 200 trees, learning rate 0.1 and depth 3, a fully connected deep neural network trained with CE (DNN-CE), \rev{and} the same network trained with BOLT (DNN-BOLT)\rev{.} \rev{The} data-driven $\varepsilon$-ball \rev{Bayes}\rev{-}\rev{error} estimator of \cite{noshad2019learning} \rev{is} \rev{shown} \rev{as} \rev{a} \rev{separate} \rev{reference}. The DNN has hidden layers of sizes 1024 and 512, each followed by ReLU, batch normalization, and dropout with probability 0.2, and is trained for 20 epochs using Adam with learning rate $10^{-3}$. For the $\varepsilon$-ball estimator, we use a Chebyshev ensemble with three nodes and select the base radius from the 95th percentile of the $k=50$ nearest-neighbor distance. The analytic \ac{ber} is computed from \eqref{BayesError} using the known Gaussian densities.

Figure~\ref{fig:synthetic} reports the error rate as a function of $n$. DNN-BOLT outperforms DNN-CE from $n=800$ onward. At $n=25600$, DNN-BOLT reaches an error of $0.285$, compared with $0.299$ for DNN-CE, while the $\varepsilon$-ball estimate is $0.281$. RF and GB remain farther from the analytic \ac{ber}, with errors of $0.332$ and $0.339$, respectively. \rev{At} \rev{the} \rev{largest} \rev{sample} \rev{size}\rev{,} \rev{DNN}\rev{-}BOLT \rev{is} \rev{closest} to the \rev{analytic} \ac{ber} \rev{among} \rev{these} \rev{four} \rev{classifiers}.

\begin{figure}[t]
  \centering
\begin{tikzpicture}[font=\footnotesize]
  \begin{axis}[
    width=0.66\linewidth,
    height=0.45\linewidth,
    tick label style={font=\scriptsize},
    label style={font=\footnotesize},
    title style={font=\footnotesize},
    xlabel={training sample size},
    ylabel={misclassification error},
    xmin=0, xmax=26000,
    ymin=0.25, ymax=0.50,
    xtick={0,5000,10000,15000,20000,25000},
    ytick={0.25,0.30,0.35,0.40,0.45,0.50},
    grid=both,
    grid style={gray!20},
    legend cell align=left,
    legend pos=north east,
    legend style={font=\scriptsize},
    every mark/.append style={solid},
    error bars/error bar style={line width=0.4pt}
  ]
\addplot+[color=green!60!black, mark=triangle*, mark options={fill=green!60!black}]
  table[x=n, y=err] {
    n err errminus errplus
    400 0.452 0.013753 0.013827
    800 0.351 0.013110 0.013339
    1600 0.330 0.012899 0.013160
    3200 0.314 0.012717 0.013003
    6400 0.314 0.012717 0.013003
    12800 0.301 0.012557 0.012863
    25600 0.285 0.012344 0.012674
  };
\addlegendentry{DNN-BOLT}

\addplot+[color=blue, mark=square*, mark options={fill=blue}]
  table[x=n, y=err] {
    n err errminus errplus
    400 0.376 0.013326 0.013516
    800 0.381 0.013365 0.013547
    1600 0.348 0.013082 0.013315
    3200 0.333 0.012931 0.013187
    6400 0.341 0.013013 0.013257
    12800 0.320 0.012787 0.013064
    25600 0.299 0.012532 0.012840
  };
\addlegendentry{DNN-CE}

\addplot+[color=red, mark=triangle, mark options={fill=red}]
  table[x=n, y=err] {
    n err errminus errplus
    400 0.464 0.013790 0.013845
    800 0.438 0.013699 0.013794
    1600 0.417 0.013598 0.013725
    3200 0.386 0.013402 0.013577
    6400 0.383 0.013380 0.013559
    12800 0.361 0.013201 0.013415
    25600 0.332 0.012920 0.013178
  };
\addlegendentry{RF}

\addplot+[color=violet, mark=diamond*, mark options={fill=violet}]
  table[x=n, y=err] {
    n err errminus errplus
    400 0.481 0.013829 0.013858
    800 0.459 0.013776 0.013839
    1600 0.413 0.013576 0.013709
    3200 0.385 0.013394 0.013571
    6400 0.364 0.013227 0.013436
    12800 0.351 0.013110 0.013339
    25600 0.339 0.012993 0.013240
  };
\addlegendentry{GB}

\addplot+[color=orange, mark=o, mark options={fill=orange}]
  table[x=n, y=err] {
    n err errminus errplus
    400 0.405 0.013529 0.013674
    800 0.374 0.013310 0.013504
    1600 0.342 0.013023 0.013266
    3200 0.318 0.012764 0.013044
    6400 0.297 0.012506 0.012817
    12800 0.290 0.012412 0.012735
    25600 0.281 0.012287 0.012623
  };
\addlegendentry{$\varepsilon$-ball estimate}

\addplot[black, dashed, very thick] coordinates {(0,0.265) (26000,0.265)};
\addlegendentry{BER $\approx 0.265$}

  \end{axis}
\end{tikzpicture}
  \caption{\rev{Classifier} \rev{errors} \rev{and} \rev{the} $\rev{\varepsilon}$\rev{-}\rev{ball} \rev{Bayes}\rev{-}error \rev{estimate} versus training-set size for the four-class task in a $d=20$\rev{-}dimensional space.}
  \label{fig:synthetic}
\end{figure}

\subsection{Experiments with Realistic Data}
\label{sec:experiments_real}
Throughout these experiments, we train neural architectures on five datasets, namely CIFAR-10~\cite{krizhevsky2009learning}, MNIST~\cite{lecun1998gradient}, Fashion-MNIST~\cite{xiao2017fashion}, IMDb sentiment~\cite{maas2011learning}, and SST-2 sentiment~\cite{socher2013recursive}. The first three are multiclass image-classification datasets, while IMDb and SST-2 are binary sentiment-analysis datasets. All models are evaluated on the test split, except for SST-2, for which validation accuracy is reported because the standard GLUE test labels are unavailable.

\paragraph*{Models}
For CIFAR-10, we train ResNet-18~\cite{he2016deep} for up to 100 epochs with batch size 256. For MNIST and Fashion-MNIST, we train a CNN with two convolutional layers, ReLU activations, max-pooling, and two fully connected layers. The CNN has approximately $10^6$ parameters. On MNIST, it is trained for up to 50 epochs with batch size 64 and early stopping selects epoch 20. On Fashion-MNIST, it is trained for up to 100 epochs and early stopping selects epoch 80. For the textual tasks, we fine-tune BERT-base~\cite{devlin2018bert} with a linear classification head and dropout. On IMDb, the model is fine-tuned for up to 10 epochs with batch size 32 and early stopping selects epoch 8. On SST-2, the model is fine-tuned for three epochs with maximum sequence length 128 and batch size 32.

We train all models with both BOLT and CE using the same architecture, optimization schedule, and data split. Unless otherwise specified, Adam is used with $\beta_1=0.9$ and $\beta_2=0.999$, together with a cosine learning-rate schedule and a five-epoch warm-up. Each configuration is evaluated over five seeds and the mean is reported.

\paragraph*{Accuracy Results}
Table~\ref{tab:accuracy} reports the classification accuracy obtained by BOLT and CE. Results for linear discriminant analysis (LDA) and random forest (RF) are further included as classical non-neural baselines. The results show that BOLT obtains higher \rev{mean} accuracy than CE across \rev{these} \rev{CNN}\rev{,} \rev{ResNet} and \rev{BERT} experiments.

Figure~\ref{fig:learning_curves} further connects these results to the learning dynamics by showing curves averaged over five seeds. The LDA and RF baselines, which do not require iterative neural-network training, are shown as dashed horizontal lines. The model-free $\varepsilon$-ball estimator of \cite{noshad2019learning} is also shown on MNIST \rev{as} \rev{an} \rev{estimate} \rev{of} \rev{the} \rev{Bayes} accuracy. This estimator is applied directly to the normalized $28\times28$ pixel vectors, i.e., in a 784-dimensional input space rather than a learned latent space. Its high-dimensional computation is expensive. Thus, we use it only on MNIST as a reference.

\begin{table}[t]
\centering
\caption{Classification accuracy (\%) for image-classification and binary sentiment-analysis tasks with BOLT and CE. \rev{Neural} \rev{entries} \rev{correspond} \rev{to} \rev{the} \rev{final} \rev{epochs} \rev{shown} \rev{in} \rev{Figure}~\rev{\ref{fig:learning_curves}}\rev{,} \rev{except} \rev{SST}\rev{-}\rev{2}\rev{,} \rev{which} \rev{uses} \rev{its} \rev{stated} \rev{three}\rev{-}\rev{epoch} \rev{run}\rev{.}}
\label{tab:accuracy}
\begin{tabular}{lccccc}
\toprule
Method & CIFAR-10 & MNIST & F-MNIST & IMDb & SST-2 \\
\midrule
LDA             & 42.0 & 87.9 & 80.4 & 86.8 & -- \\
RF              & 48.0 & 95.6 & 88.0 & 80.3 & -- \\
CE              & 91.9 & \rev{99.3} & \rev{91.6} & \rev{91.8} & 89.9 \\
BOLT            & \textbf{93.3} & \textbf{99.5} & \textbf{\rev{92.6}} & \textbf{92.8} & \textbf{91.3} \\
\bottomrule
\end{tabular}
\end{table}

\paragraph*{Computational Complexity}
Table~\ref{tab:cost} merges the computational measurements for CIFAR-10 and MNIST. The per-epoch wall-clock values are rounded measurements on one NVIDIA A100 GPU. The reported total protocol times are therefore approximate products of the rounded per-epoch value and the number of executed epochs. The $\varepsilon$-ball row is a directly measured one-shot runtime.
\begin{table}[!t]
  \centering
  \caption{Computational cost for the CIFAR-10 and MNIST experiments. Total neural-network times are approximate because they are calculated from rounded per-epoch measurements.}
  \label{tab:cost}
  \scriptsize
  \setlength{\tabcolsep}{3.5pt}
  \begin{tabular}{llcccc}
    \toprule
    Dataset & Method & FLOPs/image (G) & Time/epoch (s) & Executed epochs & Total time (s) \\
    \midrule
    CIFAR-10 & ResNet-18 + CE   & 0.041 & 12.0 & 100 & $\approx1200$ \\
    CIFAR-10 & ResNet-18 + BOLT & 0.041 & 13.0 & 100 & $\approx1300$ \\
    MNIST & CNN + CE             & -- & 3.6 & 20 & $\approx72$ \\
    MNIST & CNN + BOLT           & -- & 4.0 & 20 & $\approx80$ \\
    MNIST & $\varepsilon$-ball  & -- & -- & one shot & 210 \\
    \bottomrule
  \end{tabular}
\end{table}
BOLT has the same reported FLOPs as CE on CIFAR-10 and introduces only a modest wall-clock overhead: approximately $8\%$ per epoch on CIFAR-10 and $11\%$ on MNIST. The one-shot $\varepsilon$-ball computation takes 210 seconds on MNIST, which is more than twice the complete 20-epoch BOLT training time under this setup. Thus, BOLT provides a trainable BER-guided objective at a computational cost close to conventional CE training.

\begin{figure}[t!]
  \centering
  \centering
\begin{tikzpicture}[font=\footnotesize, spy using outlines={rectangle, magnification=3, connect spies}]
\begin{groupplot}[
  group style={group size=2 by 2, horizontal sep=0.06\linewidth, vertical sep=0.08\linewidth},
  width=0.42\linewidth, height=0.30\linewidth,
  xlabel={epoch}, ylabel={accuracy (\%)},
  tick label style={font=\scriptsize},
  label style={font=\footnotesize},
  title style={font=\footnotesize},
  grid=both, grid style={opacity=.12},
  legend style={font=\scriptsize, /tikz/every even column/.style={column sep=3pt}},legend pos=south east
]

\nextgroupplot[title={CIFAR-10 (ResNet-18)}, ymin=35, ymax=94.5, xmin=0, xmax=100]
    \addplot[blue!60, thick] table[row sep=\\, x index=0, y index=1] {
    epoch accuracy \\
    1 10.88 \\
    2 57.76 \\
    3 65.16 \\
    4 72.31 \\
    5 75.20 \\
    6 77.75 \\
    7 79.09 \\
    8 79.11 \\
    9 80.99 \\
    10 82.25 \\
    11 82.38 \\
    12 82.32 \\
    13 84.69 \\
    14 82.11 \\
    15 85.04 \\
    16 85.80 \\
    17 86.01 \\
    18 86.04 \\
    19 86.30 \\
    20 86.23 \\
    21 86.30 \\
    22 88.15 \\
    23 86.80 \\
    24 86.74 \\
    25 86.65 \\
    26 88.05 \\
    27 86.08 \\
    28 87.55 \\
    29 87.53 \\
    30 88.02 \\
    31 88.54 \\
    32 87.68 \\
    33 87.58 \\
    34 88.78 \\
    35 88.15 \\
    36 88.83 \\
    37 88.89 \\
    38 88.78 \\
    39 88.71 \\
    40 88.25 \\
    41 89.57 \\
    42 89.43 \\
    43 88.46 \\
    44 88.94 \\
    45 89.61 \\
    46 88.83 \\
    47 88.65 \\
    48 88.68 \\
    49 88.84 \\
    50 89.01 \\
    51 88.90 \\
    52 91.33 \\
    53 91.44 \\
    54 91.58 \\
    55 91.56 \\
    56 91.60 \\
    57 91.40 \\
    58 91.63 \\
    59 91.67 \\
    60 91.44 \\
    61 91.59 \\
    62 91.57 \\
    63 91.60 \\
    64 91.82 \\
    65 91.53 \\
    66 91.75 \\
    67 91.54 \\
    68 91.34 \\
    69 91.69 \\
    70 91.84 \\
    71 91.95 \\
    72 91.52 \\
    73 91.61 \\
    74 91.56 \\
    75 91.67 \\
    76 91.75 \\
    77 91.79 \\
    78 91.58 \\
    79 91.70 \\
    80 91.75 \\
    81 91.80 \\
    82 91.62 \\
    83 91.52 \\
    84 91.63 \\
    85 91.54 \\
    86 91.83 \\
    87 91.66 \\
    88 91.44 \\
    89 91.80 \\
    90 91.76 \\
    91 91.72 \\
    92 91.84 \\
    93 91.61 \\
    94 91.70 \\
    95 91.71 \\
    96 91.67 \\
    97 91.88 \\
    98 91.60 \\
    99 91.62 \\
    100 91.91 \\
};
\addlegendentry{CE}

\addplot[red, thick] table[row sep=\\, x index=0, y index=1] {
    epoch bayes_accuracy \\
    1 10.88 \\
    2 60.80 \\
    3 69.33 \\
    4 69.61 \\
    5 68.99 \\
    6 73.15 \\
    7 73.75 \\
    8 75.02 \\
    9 77.65 \\
    10 77.14 \\
    11 78.19 \\
    12 79.01 \\
    13 79.48 \\
    14 80.67 \\
    15 81.28 \\
    16 81.16 \\
    17 82.34 \\
    18 82.11 \\
    19 82.20 \\
    20 82.32 \\
    21 82.59 \\
    22 82.69 \\
    23 82.63 \\
    24 82.81 \\
    25 85.85 \\
    26 86.23 \\
    27 86.83 \\
    28 86.65 \\
    29 88.42 \\
    30 86.85 \\
    31 87.23 \\
    32 86.17 \\
    33 88.78 \\
    34 88.34 \\
    35 88.24 \\
    36 88.76 \\
    37 88.99 \\
    38 88.94 \\
    39 88.27 \\
    40 89.58 \\
    41 88.75 \\
    42 88.67 \\
    43 88.96 \\
    44 88.62 \\
    45 89.69 \\
    46 91.59 \\
    47 91.86 \\
    48 91.75 \\
    49 92.16 \\
    50 91.27 \\
    51 91.77 \\
    52 92.01 \\
    53 91.81 \\
    54 92.21 \\
    55 91.71 \\
    56 92.03 \\
    57 91.28 \\
    58 91.75 \\
    59 92.40 \\
    60 92.54 \\
    61 92.32 \\
    62 92.25 \\
    63 91.33 \\
    64 92.28 \\
    65 92.73 \\
    66 92.49 \\
    67 92.71 \\
    68 92.65 \\
    69 92.87 \\
    70 92.48 \\
    71 92.55 \\
    72 92.64 \\
    73 92.43 \\
    74 92.68 \\
    75 92.99 \\
    76 92.77 \\
    77 92.81 \\
    78 92.98 \\
    79 92.71 \\
    80 93.05 \\
    81 92.87 \\
    82 92.97 \\
    83 92.98 \\
    84 93.00 \\
    85 93.05 \\
    86 93.11 \\
    87 93.11 \\
    88 93.22 \\
    89 93.18 \\
    90 93.15 \\
    91 93.29 \\
    92 93.29 \\
    93 93.19 \\
    94 93.23 \\
    95 93.21 \\
    96 93.16 \\
    97 93.27 \\
    98 93.19 \\
    99 93.29 \\
    100 93.25 \\
};
\addlegendentry{BOLT}
\addplot[dashed, purple]  coordinates {(0,42.0) (100,42.0)}; \addlegendentry{LDA}
\addplot[dashed, green] coordinates {(0,48.0) (100,48.0)};\addlegendentry{RF} 

\coordinate (spypoi) at (axis cs:92,92.2);
\coordinate (spypo)   at (rel axis cs:0.7,0.7); 
\spy[size=1.6cm,draw,fill=white] on (spypoi) in node at (spypo);
\nextgroupplot[title={MNIST (SimpleCNN)}, ymin=85, ymax=100.0, xmin=1, xmax=20]

\addplot+[blue, thick, mark=none] table[row sep=\\]{x y\\
1 98.76 \\
2 99.06 \\
3 99.10 \\
4 99.19 \\
5 99.16 \\
6 99.34 \\
7 99.23 \\
8 99.34 \\
9 99.18 \\
10 99.41 \\
11 99.31 \\
12 98.97 \\
13 99.27 \\
14 99.37 \\
15 99.19 \\
16 99.28 \\
17 99.38 \\
18 99.39 \\
19 99.43 \\
20 99.25 \\
};\addlegendentry{CE}
\addplot+[red, thick, mark=none] table[row sep=\\]{x y\\
1 98.81 \\
2 98.93 \\
3 99.01 \\
4 99.22 \\
5 99.20 \\
6 99.40 \\
7 99.40 \\
8 99.36 \\
9 99.29 \\
10 99.33 \\
11 99.46 \\
12 99.45 \\
13 99.49 \\
14 99.47 \\
15 99.43 \\
16 99.49 \\
17 99.45 \\
18 99.48 \\
19 99.47 \\
20 99.50 \\
};\addlegendentry{BOLT}
\addplot[dashed, purple]  coordinates {(0,87.9) (20,87.9)}; \addlegendentry{LDA}
\addplot[dashed, green] coordinates {(0,95.6) (20,95.6)}; \addlegendentry{RF}
\addplot[dotted, teal, thick] coordinates {(0,99.86) (20,99.86)}; \addlegendentry{\rev{$\varepsilon$-ball estimate}}


\coordinate (A) at (axis cs:19,99.4);
\coordinate (B)   at (rel axis cs:0.3,0.65); 
\spy[size=1.6cm,draw,fill=white] on (A) in node at (B);

\nextgroupplot[title={Fashion-MNIST (SimpleCNN)}, ymin=60, ymax=93.5, xmin=0, xmax=80]
\addplot+[blue, thick, mark=none] table[row sep=\\]{x y\\
1 88.26 \\
2 89.65 \\
3 89.73 \\
4 89.89 \\
5 89.42 \\
6 90.03 \\
7 90.46 \\
8 90.20 \\
9 90.95 \\
10 91.03 \\
11 91.31 \\
12 91.17 \\
13 91.16 \\
14 91.38 \\
15 91.01 \\
16 91.14 \\
17 91.24 \\
18 91.40 \\
19 91.54 \\
20 91.89 \\
21 91.50 \\
22 91.56 \\
23 91.86 \\
24 91.69 \\
25 91.59 \\
26 91.64 \\
27 91.86 \\
28 92.13 \\
29 91.92 \\
30 91.87 \\
31 91.33 \\
32 91.56 \\
33 91.99 \\
34 91.58 \\
35 91.74 \\
36 91.57 \\
37 91.48 \\
38 91.94 \\
39 92.16 \\
40 92.15 \\
41 92.03 \\
42 91.70 \\
43 91.96 \\
44 91.83 \\
45 91.80 \\
46 91.59 \\
47 91.82 \\
48 91.86 \\
49 91.86 \\
50 91.81 \\
51 92.00 \\
52 92.01 \\
53 91.46 \\
54 91.89 \\
55 91.99 \\
56 91.81 \\
57 91.93 \\
58 91.78 \\
59 91.53 \\
60 91.90 \\
61 91.67 \\
62 91.80 \\
63 91.53 \\
64 92.05 \\
65 91.88 \\
66 91.81 \\
67 91.47 \\
68 91.94 \\
69 91.89 \\
70 91.97 \\
71 91.93 \\
72 92.16 \\
73 91.93 \\
74 91.90 \\
75 92.07 \\
76 91.87 \\
77 91.70 \\
78 91.68 \\
79 91.64 \\
80 91.60 \\
};\addlegendentry{CE}

\addplot+[red, thick, mark=none] table[row sep=\\]{x y\\
1 83.45 \\
2 84.52 \\
3 84.13 \\
4 84.74 \\
5 85.02 \\
6 83.50 \\
7 85.24 \\
8 85.60 \\
9 85.17 \\
10 85.15 \\
11 85.46 \\
12 85.29 \\
13 85.58 \\
14 85.84 \\
15 85.46 \\
16 85.69 \\
17 85.96 \\
18 85.97 \\
19 85.47 \\
20 86.07 \\
21 86.11 \\
22 86.04 \\
23 85.74 \\
24 86.00 \\
25 86.00 \\
26 85.51 \\
27 85.48 \\
28 85.80 \\
29 86.09 \\
30 85.75 \\
31 86.14 \\
32 85.96 \\
33 86.09 \\
34 85.68 \\
35 85.97 \\
36 86.13 \\
37 86.05 \\
38 86.32 \\
39 86.27 \\
40 86.47 \\
41 86.29 \\
42 86.04 \\
43 86.35 \\
44 86.25 \\
45 88.97 \\
46 90.49 \\
47 91.12 \\
48 91.18 \\
49 91.35 \\
50 91.44 \\
51 91.30 \\
52 91.14 \\
53 92.10 \\
54 91.93 \\
55 91.89 \\
56 91.68 \\
57 91.80 \\
58 91.93 \\
59 92.00 \\
60 91.97 \\
61 92.04 \\
62 91.66 \\
63 91.77 \\
64 91.88 \\
65 91.88 \\
66 91.87 \\
67 92.13 \\
68 92.05 \\
69 92.12 \\
70 92.43 \\
71 92.33 \\
72 92.58 \\
73 92.37 \\
74 92.31 \\
75 92.34 \\
76 92.58 \\
77 92.42 \\
78 92.22 \\
79 92.58 \\
80 92.58 \\
};\addlegendentry{BOLT}

\addplot[dashed, purple]  coordinates {(0,80.4) (80,80.4)}; \addlegendentry{LDA}
\addplot[dashed, green] coordinates {(0,88.0) (80,88.0)}; \addlegendentry{RF}
\coordinate (spypoint) at (axis cs:75,92.2);
\coordinate (spypos)   at (rel axis cs:0.7,0.7); 
\spy[size=1.6cm,draw,fill=white] on (spypoint) in node at (spypos);
\nextgroupplot[title={IMDb (BERT-base)}, ymin=80, ymax=95, xmin=1, xmax=8]
\addplot+[blue, thick, mark=none] table[row sep=\\]{x y\\
1  90.93 \\ 2 91.16 \\ 3 92.08 \\ 4 90.68 \\ 5 91.72 \\ 6 91.99 \\ 7 92.12 \\ 8 91.77 \\ };\addlegendentry{CE}
\addplot+[red, thick, mark=none] table[row sep=\\]{x y\\
1  88.75\\ 2 90.54\\ 3 90.03\\ 4 92.04\\ 5 92.34\\ 6 92.46\\ 7 92.73\\ 8 92.76\\ };\addlegendentry{BOLT}
\addplot[dashed, green]  coordinates {(1,80.3) (8,80.3)}; \addlegendentry{RF}
\addplot[dashed, purple] coordinates {(1,86.8) (8,86.8)}; \addlegendentry{LDA}

\end{groupplot}
\end{tikzpicture}

  \caption{Learning curves averaged over five seeds for CE and BOLT. Dashed horizontal lines show the non-iterative baselines where they lie within the displayed range, and the MNIST dotted line marks the \rev{estimated} \rev{Bayes} accuracy \rev{from} the $\varepsilon$-ball estimator.}
  \label{fig:learning_curves}
\end{figure}

\subsection{Pruning and Model Separation}
\label{sec:pruning_separation}
\subsubsection*{Pruning}
For this experiment, we consider training a CNN on the Fashion-MNIST dataset. The model has three convolutional blocks with 32, 64, and 128 filters. Each block performs convolution, batch normalization, ReLU activation, and max-pooling. The classifier head has two fully connected layers of sizes 128 and 64, followed by a 10-class output layer. We train the same architecture with BOLT and CE.

We then apply global unstructured magnitude pruning to the convolutional and fully connected weights \cite{han2015learning}. Bias terms and batch-normalization parameters are not pruned. At epochs 10, 20, and 30, we remove 20\% of the weights that remain unpruned. After the three stages, the effective sparsity is $1-(0.8)^3=48.8\%$. Training continues with the same loss after every pruning step so that the remaining weights can be fine-tuned.

Table~\ref{tab:pruning_results} reports the first post-pruning checkpoint and the final checkpoint after the complete pruning-\rev{and}-fine-tuning protocol. CE is initially more accurate after the first pruning step. Nevertheless, as training proceeds, the BOLT-trained model recovers and reaches 92.74\% accuracy, compared with 92.47\% for CE at the same 48.8\% sparsity. This result indicates that BOLT remains compatible with iterative magnitude pruning and adapts effectively during post-pruning fine-tuning. 

\begin{table}[t]
\centering
\caption{Fashion-MNIST pruning results at the first retained post-pruning checkpoint and the final pruning-\rev{and}-fine-tuning checkpoint.}
\label{tab:pruning_results}
\begin{tabular}{lccc}
\toprule
Checkpoint & Sparsity (\%) & BOLT (\%) & CE (\%) \\
\midrule
Epoch 10 & 20.0 & 85.23 & \textbf{91.29} \\
Epoch 50 & 48.8 & \textbf{92.74} & 92.47 \\
\bottomrule
\end{tabular}
\end{table}

\subsubsection*{Model Separation}
To better understand the effect of BOLT on the learned representation, we compare the embeddings produced by the CNN model in Section~\ref{sec:experiments_real} when it is trained on MNIST with BOLT or CE. Figure~\ref{fig:tsne_comparison} shows $t$-distributed stochastic neighbor embedding ($t$-SNE) visualizations of the test embeddings after 10 training epochs. The embeddings are extracted from the penultimate layer. The \rev{projections} \rev{provide} \rev{a} \rev{qualitative} \rev{comparison} \rev{of} \rev{the} \rev{learned} \rev{representations}\rev{;} \rev{each} \rev{panel} \rev{uses} \rev{a} \rev{separate} \rev{two}-\rev{dimensional} \rev{projection}.

\begin{figure}[t]
    \centering
    \subfloat[Cross-entropy loss]{
    \includegraphics[width=0.45\textwidth]{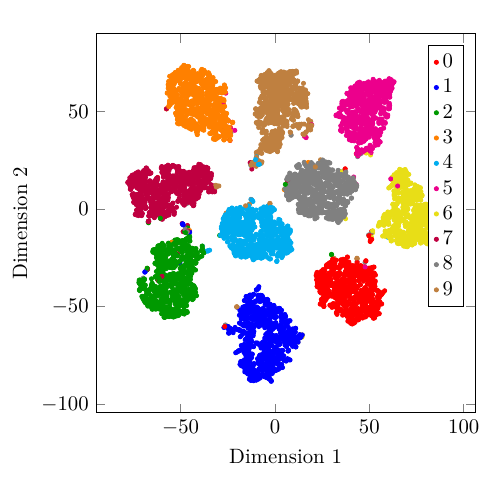}
        \label{fig:tsne_crossentropy}
    }
    \hfill
    \subfloat[BOLT loss]{
    \includegraphics[width=0.45\textwidth]{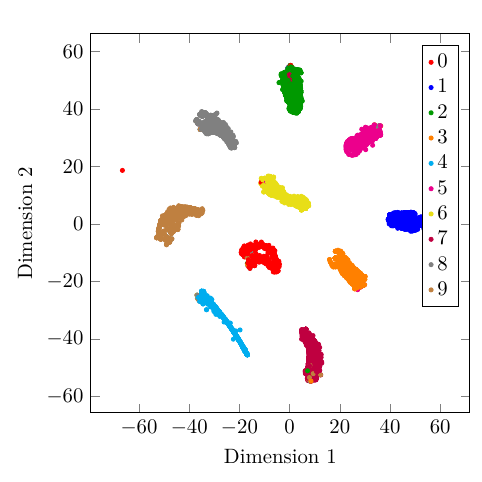}
        \label{fig:tsne_bolt}
    }
    \caption{\texorpdfstring{$t$}{t}-SNE visualizations of MNIST test embeddings learned by the CNN with CE and BOLT after 10 epochs.}
    \label{fig:tsne_comparison}
\end{figure}

\section{Conclusions}
\label{sec:conclusion}
In this paper, we established a new training framework based on a universal bound on the \ac{ber}. The proposed BOLT framework trains a learning model by minimizing the \rev{empirical} \rev{BOLT} \rev{risk}, offering an alternative training approach to risk minimization with conventional loss functions. Our numerical experiments demonstrated that, for the reported datasets and matched architectures, BOLT \rev{often} \rev{attained} \rev{lower} \rev{classification} \rev{error} \rev{than} the corresponding cross-entropy baselines. Extending the BOLT framework to a wider set of learning tasks, such as regression, is an interesting direction for future work. Work in this direction is currently ongoing.

\appendices

\section{Proof of Lemma~\ref{thm:bayes_error_f_divergence}}
\label{appendix:proof_uniform_prior_fdiv}
\noindent
Starting from \eqref{Decompose}, the \ac{ber} for a binary classification with priors \( \pi_1 \) and \( \pi_2 = 1 - \pi_1 \) is given by
\begin{align}
&\varepsilon_{\mathrm{Bayes}}
   = \pi_2 \Big(  1 \;\;  - \int_{\{\rev{x}\rev{:}\rev{p}\rev{_2}\rev{(}\rev{x}\rev{)}\rev{>}\rev{0}\}}
    \max\Biggl\{ 0,\; 1 - \frac{\pi_1\,p_1(x)}
    {\pi_2\,p_2(x)} \Biggr\}
    p_2(x)\,\dif x \Big)\rev{.}
\end{align}
\rev{The} \rev{omitted} \rev{region} \rev{contributes} \rev{zero} \rev{to} $\rev{[}\rev{\pi}\rev{_2}\rev{p}\rev{_2}\rev{-}\rev{\pi}\rev{_1}\rev{p}\rev{_1}\rev{]}\rev{_+}$\rev{,} \rev{so} \rev{no} \rev{common}\rev{-}\rev{support} \rev{assumption} \rev{is} \rev{needed}\rev{.} \rev{For} a uniform prior\rev{,} \rev{this} reduces to
\begin{equation}
\label{Binary}
\varepsilon_{\mathrm{Bayes}} = \frac{1}{2} - \frac{1}{2} \int_{\{\rev{x}\rev{:}\rev{p}\rev{_2}\rev{(}\rev{x}\rev{)}\rev{>}\rev{0}\}} \max \Bigl\{ 0,\; 1 - \frac{p_1(x)}{p_2(x)} \Bigr\} \, p_2(x) \, \dif x.
\end{equation}
By replacing $f_{\mathrm{hng}}$ into the definition of the $f$-divergence, we have
\begin{align}
 D_{f_{\mathrm{hng}}}\!\bigl(P_1\,\|\,P_2\bigr) \!= \!\! \int_{\{\rev{x}\rev{:}\rev{p}\rev{_2}\rev{(}\rev{x}\rev{)}\rev{>}\rev{0}\}}
  \max\set{ 0,\! 1 - \frac{ p_1(x) }{ p_2(x) }}
  p_2(x) \dif x,
\end{align}
which concludes the proof.
\hfill\qedsymbol

\section{Proof of Lemma~\ref{hinge_fenchel}}
\label{appendix:proof_shifted_hinge_fenchel}
\noindent
For fixed $s\in\mathbb R$, define
\begin{equation}
F_s(u)\coloneqq su-f_{\mathrm{shng}}(u;t).
\label{eq:fenchel_objective_app}
\end{equation}
By definition,
\begin{equation}
f_{\mathrm{shng}}^*(s)=\sup_{u\rev{>}\rev{0}}F_s(u).
\end{equation}
We partition $\rev{(}\rev{0}\rev{,}\rev{\infty}\rev{)}$ into the two subsets $\rev{(}\rev{0}\rev{,}t\rev{]}$ and $\rev{(}t\rev{,}\rev{\infty}\rev{)}$.

For any $\rev{0}\rev{<}u\le t$, we have $f_{\mathrm{shng}}(u;t)=t-u$, and therefore
\begin{equation}
F_s(u)=(s+1)u-t.
\end{equation}
Consequently, the supremum over this subset depends on the sign of $s+1$:
\begin{equation}
\sup_{\rev{0}\rev{<}u\le t}F_s(u)=
\begin{cases}
st, & s\ge -1,\\
\rev{-}\rev{t}, & s<-1.
\end{cases}
\label{eq:fenchel_left_partition}
\end{equation}
For $s>-1$, the affine function is increasing and its supremum is attained at $u=t$. For $s=-1$, it is constant and equal to $-t=st$. For $s<-1$, \rev{its} \rev{supremum} $\rev{-}\rev{t}$ \rev{is} \rev{approached} as $u\rev{\downarrow}\rev{0}$.

For any $u>t$, we have $f_{\mathrm{shng}}(u;t)=0$, and therefore $F_s(u)=su$. In this case, the supremum depends on the sign of $s$:
\begin{equation}
\sup_{u>t}F_s(u)=
\begin{cases}
st, & s\le0,\\
+\infty, & s>0,
\end{cases}
\label{eq:fenchel_right_partition}
\end{equation}
where for $s\le0$ the value $st$ is approached as $u\downarrow t$.

Using the suprema of $F_s(u)$ on the two partitioned subsets, we can write
\begin{align}
f_{\mathrm{shng}}^*(s)
&=
\max\left\{
\sup_{\rev{0}\rev{<}u\le t}F_s(u),
\sup_{u>t}F_s(u)
\right\}\\
&=
\begin{cases}
\rev{-}\rev{t}\rev{,} & \rev{s}\rev{<}\rev{-}\rev{1}\rev{,}\\
st, & -1\le s\le0,\\
+\infty, & \rev{s}\rev{>}\rev{0},
\end{cases}
\end{align}
which concludes the proof.
\hfill\qedsymbol

\section{Proof of Theorem~\ref{thm:multiclass_upper_bound}}
\label{app:proof_multiclass_upper_bound}
\noindent
For an $m$-class classification task, the \ac{ber} is given by \eqref{Decompose}. For $\lambda\ge1$, the function in \eqref{eq:mu_ver1} satisfies the recursion
\begin{equation}
 \rev{r}_\lambda(x)=\max\{ \rev{r}_{\lambda-1}(x),\pi_\lambda p_\lambda(x)\},
\qquad  \rev{r}_0(x)=0.
\label{eq:mu_recursion_app}
\end{equation}
Since $ \rev{r}_\lambda(x)$ is nondecreasing in $\lambda$, define its increment as
\begin{equation}
\Delta_\lambda(x)\coloneqq \rev{r}_\lambda(x)- \rev{r}_{\lambda-1}(x).
\end{equation}
Using \eqref{eq:mu_recursion_app}, we can write
\begin{equation}
\Delta_\lambda(x)=\pos{\pi_\lambda p_\lambda(x)- \rev{r}_{\lambda-1}(x)}.
\label{eq:Delta_lam}
\end{equation}
Substituting this expression into \eqref{Decompose} gives the \ac{ber} as
\begin{equation}
\Bayes
=
1-\pi_1-
\sum_{\lambda=2}^m\int_{\mathcal X}\Delta_\lambda(x)\,\dif x.
\label{eq:alter}
\end{equation}

We next \rev{bound} the $\lambda$th increment through a binary subtask between $C_\lambda$ and the \rev{preceding} \rev{classes} $\{\rev{C}\rev{_1}\rev{,}\rev{\ldots}\rev{,}\rev{C}_{\rev{\lambda}\rev{-}\rev{1}}\}$\rev{.} \rev{Define} \rev{their} \rev{aggregate} \rev{prior} \rev{as}
\begin{equation}
\rev{\rho}_{\lambda-1}\coloneqq\sum_{i=1}^{\lambda-1}\pi_i\rev{.}
\end{equation}
\rev{Let} $\rev{Q}_{\rev{\lambda}\rev{-}\rev{1}}$ \rev{denote} \rev{the} \rev{conditional} \rev{distribution} \rev{of} $\rev{X}$ \rev{given} $\rev{C}\rev{\in}\{\rev{C}\rev{_1}\rev{,}\rev{\ldots}\rev{,}\rev{C}_{\rev{\lambda}\rev{-}\rev{1}}\}$\rev{,} and define
\begin{equation}
\pi_{\mathrm{total}}^{(\lambda)}
\coloneqq
\pi_\lambda+\rev{\rho}_{\lambda-1}.
\end{equation}
The normalized class probabilities for this binary subtask are
\begin{subequations}
\begin{align}
\widetilde\pi_{\lambda}
&=\frac{\pi_\lambda}{\pi_{\mathrm{total}}^{(\lambda)}},\\
\rev{\widetilde\rho}_{\lambda-1}
&=\frac{\rev{\rho}_{\lambda-1}}{\pi_{\mathrm{total}}^{(\lambda)}}.
\end{align}
\end{subequations}
\rev{These} \rev{are} \rev{the} \rev{normalized} \rev{priors} \rev{after} \rev{conditioning} \rev{on} $\rev{C}\rev{\in}\{\rev{C}\rev{_1},\rev{\ldots}\rev{,}C_\lambda\}$\rev{.} \rev{The} \rev{density} of $\rev{Q}_{\lambda-1}$ is \rev{given} by
\begin{equation}
\rev{q}_{\lambda-1}(x)
\rev{\coloneqq}
\frac{1}{\rev{\rho}_{\lambda-1}}
\sum_{i=1}^{\lambda-1}\pi_i p_i(x).
\label{eq:p_R}
\end{equation}

Let $\widetilde\varepsilon_\lambda$ denote the \ac{ber} of this normalized binary subtask. Applying \eqref{BayesError} gives
\begin{equation}
1-\widetilde\varepsilon_\lambda
=
\int_{\mathcal X}
\max\left\{
\widetilde\pi_{\lambda}p_\lambda(x),
\rev{\widetilde\rho}_{\lambda-1}\rev{q}_{\lambda-1}(x)
\right\}\dif x.
\end{equation}
Multiplying by $\pi_{\mathrm{total}}^{(\lambda)}$ and using $\max\{a,b\}=b+\pos{a-b}$, we obtain
\begin{align}
\pi_\lambda-
\pi_{\mathrm{total}}^{(\lambda)}\widetilde\varepsilon_\lambda
&=
\int_{\mathcal X}
\pos{
\pi_\lambda p_\lambda(x)-
\rev{\rho}_{\lambda-1}\rev{q}_{\lambda-1}(x)
}\dif x.
\label{eq:int1}
\end{align}

Using \eqref{eq:mu_ver1}, \eqref{eq:p_R}, and $\max_i a_i\le\sum_i a_i$, we have
\begin{equation}
 \rev{r}_{\lambda-1}(x)
\le
\sum_{i=1}^{\lambda-1}\pi_i p_i(x)
=
\rev{\rho}_{\lambda-1}\rev{q}_{\lambda-1}(x).
\end{equation}
Comparing this inequality with \eqref{eq:Delta_lam}, we can conclude that
\begin{equation}
\Delta_\lambda(x)
\ge
\pos{
\pi_\lambda p_\lambda(x)-
\rev{\rho}_{\lambda-1}\rev{q}_{\lambda-1}(x)
}.
\end{equation}
Integrating and using \eqref{eq:int1} yields
\begin{equation}
\int_{\mathcal X}\Delta_\lambda(x)\,\dif x
\ge
\pi_\lambda-
\pi_{\mathrm{total}}^{(\lambda)}\widetilde\varepsilon_\lambda.
\label{relation2}
\end{equation}
Substitution into \eqref{eq:alter} gives
\begin{equation}
\Bayes
\le
1-\pi_1-
\sum_{\lambda=2}^m
\left(
\pi_\lambda-
\pi_{\mathrm{total}}^{(\lambda)}\widetilde\varepsilon_\lambda
\right).
\label{eq:bnd_I}
\end{equation}
Thus, the multiclass \ac{ber} is bounded in terms of the BERs of the binary subtasks.

We now apply Theorem~\ref{thm:binary_upper_bound} to each subtask. For any integrable $h_\lambda:\mathcal X\to[-1,0]$,
\begin{align}
\widetilde\varepsilon_\lambda
\le{}
&\widetilde\pi_{\lambda}
+\widetilde\pi_{\lambda}
\mathbb E_{X\sim P_\lambda}[h_\lambda(X)]\\
&-
\rev{\widetilde\rho}_{\lambda-1}
\mathbb E_{X\sim \rev{Q}_{\lambda-1}}[h_\lambda(X)].
\end{align}
Multiplying by $\pi_{\mathrm{total}}^{(\lambda)}$ and rearranging gives
\begin{align}
\pi_\lambda-
\pi_{\mathrm{total}}^{(\lambda)}\widetilde\varepsilon_\lambda
\ge
\widetilde{\mathcal E}_\lambda,
\label{eq:bnd_rarg}
\end{align}
where
\begin{equation}
\widetilde{\mathcal E}_\lambda
\coloneqq
\rev{\rho}_{\lambda-1}
\mathbb E_{X\sim \rev{Q}_{\lambda-1}}[h_\lambda(X)]
-
\pi_\lambda
\mathbb E_{X\sim P_\lambda}[h_\lambda(X)].
\end{equation}
Using \eqref{eq:bnd_rarg} in \eqref{eq:bnd_I}, we have
\begin{equation}
\Bayes
\le
1-\pi_1-
\sum_{\lambda=2}^m\widetilde{\mathcal E}_\lambda.
\label{eq:multiclass_intermediate_bound_app}
\end{equation}

The \rev{mixture} \rev{definition} \rev{in} \rev{\eqref{eq:p_R}} implies
\begin{equation}
\widetilde{\mathcal E}_\lambda
=
\sum_{\gamma=1}^{\lambda-1}
\pi_\gamma\mathbb E_{X\sim P_\gamma}[h_\lambda(X)]
-
\pi_\lambda\mathbb E_{X\sim P_\lambda}[h_\lambda(X)].
\end{equation}
Therefore, after summing over $\lambda$ and regrouping by the class distribution, we can write
\begin{align}
\sum_{\lambda=2}^{m}\widetilde{\mathcal E}_\lambda
={}
&\sum_{\lambda=1}^{m}
\pi_\lambda\mathbb E_{X\sim P_\lambda}
\left[\sum_{i=\lambda+1}^{m}h_i(X)\right]\\
&-
\sum_{\lambda=2}^{m}
\pi_\lambda\mathbb E_{X\sim P_\lambda}[h_\lambda(X)].
\end{align}
Since $h_1(x)\equiv-1$, the definition of $\mathcal E_\lambda$ in Theorem~\ref{thm:multiclass_upper_bound} gives
\begin{equation}
\sum_{\lambda=1}^{m}\pi_\lambda\mathcal E_\lambda
=
\pi_1+
\sum_{\lambda=2}^{m}\widetilde{\mathcal E}_\lambda.
\end{equation}
Consequently, we can write the bound in \eqref{eq:multiclass_intermediate_bound_app} as
\begin{equation}
\Bayes
\le
1-
\sum_{\lambda=1}^{m}\pi_\lambda\mathcal E_\lambda,
\end{equation}
which completes the proof.
\hfill\qedsymbol

\section{Proof of Lemma~\ref{lem:multiclass_bayes_hinge}}
\label{app:proof_multiclass_bayes_hinge}
\noindent
Recall from \eqref{eq:mu_ver1} that
\begin{equation}
 \rev{r}_\lambda(x)=\max_{1\le j\le\lambda}\pi_jp_j(x).
\end{equation}
The telescopic decomposition of the \ac{ber} in \eqref{Decompose} is
\begin{equation}
\Bayes
=
1-\pi_1-
\sum_{\lambda=2}^{m}
\int_{\mathcal X}
\bigl( \rev{r}_\lambda(x)- \rev{r}_{\lambda-1}(x)\bigr)\dif x.
\label{eq:bayes_telescopic_app}
\end{equation}
We know from \eqref{eq:mu_recursion_app} in Appendix~\ref{app:proof_multiclass_upper_bound} that
\begin{equation}
 \rev{r}_\lambda(x)- \rev{r}_{\lambda-1}(x)
=
\pos{\pi_\lambda p_\lambda(x)- \rev{r}_{\lambda-1}(x)}.
\end{equation}
We can further write
\begin{equation}
 \rev{r}_{\lambda-1}(x)
=
\max_{1\le j<\lambda}\pi_jp_j(x).
\end{equation}
Thus,
\begin{align}
&\int_{\mathcal X}
\bigl( \rev{r}_\lambda(x)- \rev{r}_{\lambda-1}(x)\bigr)\dif x\\
&\quad=
\int_{\mathcal X}
\pos{
\pi_\lambda p_\lambda(x)-
\max_{1\le j<\lambda}\pi_jp_j(x)
}\dif x.
\label{eq:hinge_integral_app}
\end{align}
On the set $\{x:p_\lambda(x)>0\}$, the integrand in \eqref{eq:hinge_integral_app} is equivalently given by
\begin{align}
&\pos{
\pi_\lambda p_\lambda(x)-
\max_{1\le j<\lambda}\pi_jp_j(x)
}\\
&\quad=
p_\lambda(x)
\pos{
\pi_\lambda-
\max_{1\le j<\lambda}
\pi_j\frac{p_j(x)}{p_\lambda(x)}
}\\
&\quad=
p_\lambda(x)
\tzero^{(\lambda)}
\bigl(U_{1/\lambda}(x),\ldots,U_{(\lambda-1)/\lambda}(x)\bigr).
\end{align}
The complementary set $\{x:p_\lambda(x)=0\}$ contributes zero to the integral. Therefore, we can write
\begin{align}
&\int_{\mathcal X}
\bigl( \rev{r}_\lambda(x)- \rev{r}_{\lambda-1}(x)\bigr)\dif x\\
&\quad=
\int_{\mathcal X}
p_\lambda(x)
\tzero^{(\lambda)}
\bigl(U_{1/\lambda}(x),\ldots,U_{(\lambda-1)/\lambda}(x)\bigr)
\dif x\\
&\quad=
\mathbb E_{X\sim P_\lambda}
\left[
\tzero^{(\lambda)}
\bigl(U_{1/\lambda}(X),\ldots,U_{(\lambda-1)/\lambda}(X)\bigr)
\right].
\end{align}
Substituting this identity into the telescopic decomposition in \eqref{eq:bayes_telescopic_app} concludes the proof.
\hfill\qedsymbol

\section{Proof of Corollary~\ref{cor:binary_specialization}}
\label{app:proof_binary_specialization}
\noindent
For the binary problem, let $U(x)=p_1(x)/p_2(x)$ \rev{on} $\{\rev{p}\rev{_2}\rev{>}\rev{0}\}$ and $\rev{t}=\pi_2/\pi_1$. The BOLT population risk reduces to
\begin{equation}
\mathcal L_{\mathrm{BOLT}}(h)
=
\pi_2-
\left(
\pi_1\mathbb E_{X\sim P_1}[h(X)]
-
\pi_2\mathbb E_{X\sim P_2}[h(X)]
\right),
\label{eq:bolt-functional}
\end{equation}
for any measurable $h:\mathcal X\to[-1,0]$. This functional upper bounds the \ac{ber}, as stated in Theorem~\ref{thm:binary_upper_bound}.

For fixed $x$, define
\begin{equation}
c(x)\coloneqq\pi_1p_1(x)-\pi_2p_2(x).
\end{equation}
Minimizing \eqref{eq:bolt-functional} pointwise is equivalent to maximizing $c(x)h(x)$ over $h(x)\in[-1,0]$, whose solution is
\begin{equation}
h^\star(x)=
\begin{cases}
-1, & c(x)<0,\\
0, & c(x)\ge0.
\end{cases}
\end{equation}
This solution can be equivalently written as
\begin{equation}
h^\star(x)=
\begin{cases}
-1, & \rev{p}\rev{_2}\rev{(}\rev{x}\rev{)}\rev{>}\rev{0}\rev{\text{ and }}U(x)<\rev{t},\\
0, & \rev{\text{otherwise}}.
\end{cases}
\end{equation}
Substitution into \eqref{eq:bolt-functional} gives
\begin{align}
\mathcal L_{\mathrm{BOLT}}(h^\star)
&=
\pi_2-
\int_{\mathcal X}
\pos{\pi_2p_2(x)-\pi_1p_1(x)}\dif x,\\
&=
1-
\int_{\mathcal X}
\max\{\pi_1p_1(x),\pi_2p_2(x)\}\dif x,\\
&=
\Bayes.
\end{align}

To show the likelihood-ratio representation in \eqref{eq:bayes_error_expectation}, consider the set $\{x:p_2(x)>0\}$, on which we have
\begin{equation}
\pos{\pi_2p_2(x)-\pi_1p_1(x)}
=
p_2(x)\pos{\pi_2-\pi_1U(x)}.
\end{equation}
The complementary subset contributes zero. Hence, with $t_0(u)=\pos{\pi_2-\pi_1u}$, we have
\begin{equation}
\int_{\mathcal X}
\pos{\pi_2p_2(x)-\pi_1p_1(x)}\dif x
=
\mathbb E_{X\sim P_2}[t_0(U(X))].
\end{equation}
Therefore, we can write
\begin{equation}
\Bayes
=
\pi_2-
\mathbb E_{X\sim P_2}[t_0(U(X))].
\end{equation}
Noting that $h_2^\star(x)=0$ is equivalent to $\pi_1p_1(x)\ge\pi_2p_2(x)$ and $h_2^\star(x)=-1$ is equivalent to $\pi_2p_2(x)>\pi_1p_1(x)$, we can further conclude the stated MAP rule.
\hfill\qedsymbol

\section{Proof of Theorem~\ref{thm:multiclass_optimizer_tightness}}
\label{app:proof_multiclass_optimizer_tightness}
\noindent
Consider the expanded BOLT population risk in \eqref{eq:bolt_population_risk_expanded}, with dependence on $\theta$ suppressed. Since $h_1(x)\equiv-1$, its contribution \rev{to} \rev{the} \rev{risk} is the constant $\rev{-}\pi_1$. For each $\lambda\ge2$, the function $h_\lambda(x)$ appears \rev{negatively} in the loss terms \rev{for} classes $C_1,\ldots,C_{\lambda-1}$ and \rev{positively} in the term \rev{for} $C_\lambda$.
Consequently, we can write
\begin{align}
\mathcal L(h)
&=
1-\pi_1-
\sum_{\lambda=2}^{m}
\int_{\mathcal X}
\bigl(S_{\lambda-1}(x)-\pi_\lambda p_\lambda(x)\bigr)
 h_\lambda(x)\dif x.
\label{eq:lbolt-sum-collected}
\end{align}
For fixed $x$ and $\lambda$, minimizing the corresponding integrand over $h_\lambda(x)\in[-1,0]$ is a linear optimization problem. The pointwise minimizer is
\begin{equation}
h_\lambda^\star(x)=
\begin{cases}
-1, & \pi_\lambda p_\lambda(x)>S_{\lambda-1}(x),\\
0, & \pi_\lambda p_\lambda(x)\le S_{\lambda-1}(x),
\end{cases}
\end{equation}
with either endpoint admissible on a tie. Substituting this minimizer into \eqref{eq:lbolt-sum-collected} gives
\begin{equation}
\mathcal L^\star
=
1-\pi_1-
\sum_{\lambda=2}^{m}
\int_{\mathcal X}
\pos{\pi_\lambda p_\lambda(x)-S_{\lambda-1}(x)}\dif x,
\end{equation}
which proves \eqref{eq:lbolt_sum_optimizer_value}.

We now compare this value with the \ac{ber}. By Lemma~\ref{lem:multiclass_bayes_hinge}, equivalently by \eqref{eq:bayes_telescopic_app},
\begin{equation}
\Bayes
=
1-\pi_1-
\sum_{\lambda=2}^{m}
\int_{\mathcal X}
\pos{\pi_\lambda p_\lambda(x)- \rev{r}_{\lambda-1}(x)}\dif x.
\label{eq:bayes-max-hinge-app}
\end{equation}
Subtracting \eqref{eq:bayes-max-hinge-app} from \eqref{eq:lbolt_sum_optimizer_value} yields
\begin{align}
\Delta_{\mathrm{ms}}
=
\sum_{\lambda=2}^{m}
\int_{\mathcal X}
\Bigl(
&\pos{\pi_\lambda p_\lambda(x)- \rev{r}_{\lambda-1}(x)}\\
&-
\pos{\pi_\lambda p_\lambda(x)-S_{\lambda-1}(x)}
\Bigr)\dif x.
\label{eq:gap_difference_positive_parts_app}
\end{align}
Because $ \rev{r}_{\lambda-1}(x)\le S_{\lambda-1}(x)$, every integrand in \eqref{eq:gap_difference_positive_parts_app} is nonnegative. For any $a\in\mathbb R$ and $b\ge0$,
\begin{equation}
\pos{a}-\pos{a-b}=\min\{\pos{a},b\}.
\end{equation}
Setting
\begin{equation}
a=\pi_\lambda p_\lambda(x)- \rev{r}_{\lambda-1}(x),
\qquad
b=S_{\lambda-1}(x)- \rev{r}_{\lambda-1}(x),
\end{equation}
we obtain the expression in \eqref{eq:lbolt_bayes_gap_delta}. Therefore, $\Delta_{\mathrm{ms}}\ge0$. In the binary case, $S_1(x)= \rev{r}_1(x)=\pi_1p_1(x)$, and hence $\Delta_{\mathrm{ms}}=0$.
\hfill\qedsymbol

\section{Proof of Theorem~\ref{thm:max_sum_gap_size}}
\label{app:proof_max_sum_gap_size}
\noindent
\rev{Write} $\rev{p}\rev{_X}\rev{(}\rev{x}\rev{)}\rev{=}\rev{\sum_{j=1}^m}\rev{\pi}\rev{_j}\rev{p}\rev{_j}\rev{(}\rev{x}\rev{)}$ \rev{for} \rev{the} \rev{feature} \rev{density}\rev{.} We prove the two branches in \eqref{eq:max_sum_gap_upper_definition} separately and then take their minimum. Both bounds are obtained pointwise. More precisely, after fixing $x\in\mathcal X$, the first induction gives a bound proportional to the pointwise Bayes error, $S_m-r_m$, while the second induction gives the complementary bound $(m r_m-S_m)/2$. Integrating the two pointwise inequalities yields the two upper bounds on $\Delta_{\mathrm{ms}}$.

Fix $x\in\mathcal X$ and suppress the dependence on $x$. For $k\in[m]$, define
\begin{equation}
a_k\coloneqq\pi_kp_k(x),
\qquad
S_k\coloneqq\sum_{j=1}^{k}a_j,
\qquad
r_k\coloneqq\max_{1\le j\le k}a_j.
\end{equation}
Since $a_j\ge0$ for every $j$, we always have
\begin{equation}
0\le r_k\le S_k.
\label{eq:gap_sum_dominates_max}
\end{equation}
The relation $r_k\le S_k$ will be used repeatedly below; it also orders the two thresholds that define the three induction cases.

For $k\ge2$, define the pointwise gap increment
\begin{equation}
g_k
\coloneqq
\pos{a_k-r_{k-1}}-
\pos{a_k-S_{k-1}},
\end{equation}
and the accumulated pointwise gap
\begin{equation}
G_k\coloneqq\sum_{q=2}^{k}g_q.
\end{equation}
Restoring the dependence on $x$, we have $r_{k-1}(x)= \rev{r}_{k-1}(x)$. Hence, by \eqref{eq:gap_difference_positive_parts_app}, $G_m(x)$ is exactly the pointwise integrand of the max--sum gap:
\begin{equation}
\Delta_{\mathrm{ms}}
=
\int_{\mathcal X}G_m(x)\dif x.
\label{eq:gap_integral_of_Gm}
\end{equation}
Thus, $G_m(x)$ is the pointwise quantity to be bounded before integration. Moreover, because $r_{k-1}\le S_{k-1}$ and the positive-part map is nondecreasing, each $g_k$ is nonnegative.

For later use, \eqref{eq:gap_sum_dominates_max} gives the following explicit update when class $k+1$ is added:
\begin{equation}
g_{k+1}
=
\begin{cases}
0,
& a_{k+1}\le r_k,\\[2pt]
a_{k+1}-r_k,
& r_k<a_{k+1}\le S_k,\\[2pt]
S_k-r_k,
& a_{k+1}>S_k.
\end{cases}
\label{eq:gap_increment_three_cases}
\end{equation}
In every case,
\begin{equation}
S_{k+1}=S_k+a_{k+1},
\qquad
r_{k+1}=\max\{r_k,a_{k+1}\},
\qquad
G_{k+1}=G_k+g_{k+1}.
\label{eq:gap_basic_updates}
\end{equation}

\emph{First pointwise bound.}
We first prove, by a coupled induction, that for every $k\ge2$,
\begin{subequations}
\begin{align}
S_k-r_k
&\le
2^{k-2}\bigl(S_k-r_k-G_k\bigr),
\label{eq:gap_induction_first}\\
S_k
&\le
2^{k-1}\bigl(r_k-G_k\bigr).
\label{eq:gap_induction_second}
\end{align}
\end{subequations}
The first inequality will give the desired pointwise bound; the second is an auxiliary inequality needed to propagate the first one through the induction.

For the base case $k=2$, $S_1=r_1=a_1$, and therefore
\begin{equation}
g_2=\pos{a_2-a_1}-\pos{a_2-a_1}=0,
\qquad
G_2=0.
\end{equation}
Consequently, \eqref{eq:gap_induction_first} holds with equality. Also, $S_2=a_1+a_2\le2\max\{a_1,a_2\}=2r_2$, so \eqref{eq:gap_induction_second} holds as well.

Now assume that both \eqref{eq:gap_induction_first} and \eqref{eq:gap_induction_second} hold at some $k\ge2$. Since their left-hand sides are nonnegative, the induction hypotheses also imply
\begin{equation}
S_k-r_k-G_k\ge0,
\qquad
r_k-G_k\ge0.
\label{eq:gap_induction_nonnegative_residuals}
\end{equation}
We verify the two inequalities at $k+1$ in the three cases of \eqref{eq:gap_increment_three_cases}.

\emph{Case 1: $a_{k+1}\le r_k$.}
Here,
\begin{equation}
g_{k+1}=0,
\qquad
r_{k+1}=r_k,
\qquad
G_{k+1}=G_k.
\end{equation}
Using \eqref{eq:gap_induction_first}, $a_{k+1}\ge0$, and \eqref{eq:gap_induction_nonnegative_residuals}, we obtain
\begin{align}
S_{k+1}-r_{k+1}
&=(S_k-r_k)+a_{k+1}\\
&\le
2^{k-2}(S_k-r_k-G_k)+a_{k+1}\\
&\le
2^{k-1}\bigl(S_k-r_k-G_k+a_{k+1}\bigr)\\
&=
2^{k-1}\bigl(S_{k+1}-r_{k+1}-G_{k+1}\bigr).
\end{align}
This proves \eqref{eq:gap_induction_first} at $k+1$. Furthermore, $a_{k+1}\le r_k\le S_k$, where the last inequality is \eqref{eq:gap_sum_dominates_max}. Therefore, by \eqref{eq:gap_induction_second},
\begin{align}
S_{k+1}
&=S_k+a_{k+1}\\
&\le2S_k\\
&\le2^k(r_k-G_k)\\
&=2^k(r_{k+1}-G_{k+1}),
\end{align}
which proves \eqref{eq:gap_induction_second} at $k+1$.

\emph{Case 2: $r_k<a_{k+1}\le S_k$.}
In this case,
\begin{equation}
g_{k+1}=a_{k+1}-r_k,
\qquad
r_{k+1}=a_{k+1},
\qquad
G_{k+1}=G_k+a_{k+1}-r_k.
\end{equation}
The first inequality follows from the auxiliary induction hypothesis \eqref{eq:gap_induction_second}:
\begin{align}
S_{k+1}-r_{k+1}
&=S_k\\
&\le2^{k-1}(r_k-G_k)\\
&\le2^{k-1}\bigl(r_k-G_k+S_k-a_{k+1}\bigr)\\
&=2^{k-1}\bigl(S_{k+1}-r_{k+1}-G_{k+1}\bigr),
\end{align}
where the second inequality uses $S_k-a_{k+1}\ge0$. For the auxiliary inequality, $a_{k+1}\le S_k$ and
$r_{k+1}-G_{k+1}=r_k-G_k$, so
\begin{align}
S_{k+1}
&=S_k+a_{k+1}\\
&\le2S_k\\
&\le2^k(r_k-G_k)\\
&=2^k(r_{k+1}-G_{k+1}).
\end{align}

\emph{Case 3: $a_{k+1}>S_k$.}
Now,
\begin{equation}
g_{k+1}=S_k-r_k,
\qquad
r_{k+1}=a_{k+1},
\qquad
G_{k+1}=G_k+S_k-r_k.
\end{equation}
Since
\begin{equation}
S_{k+1}-r_{k+1}=S_k,
\qquad
S_{k+1}-r_{k+1}-G_{k+1}=r_k-G_k,
\end{equation}
\eqref{eq:gap_induction_second} directly gives
\begin{equation}
S_{k+1}-r_{k+1}
\le
2^{k-1}\bigl(S_{k+1}-r_{k+1}-G_{k+1}\bigr).
\end{equation}
To prove the auxiliary inequality, consider the difference between its right- and left-hand sides. Using the updates above,
\begin{align}
&2^k(r_{k+1}-G_{k+1})-S_{k+1}\\
&\quad=
2^k(a_{k+1}-G_k-S_k+r_k)-(S_k+a_{k+1})\\
&\quad=
\bigl[2^k(r_k-G_k)-2S_k\bigr]
+
(2^k-1)(a_{k+1}-S_k).
\label{eq:gap_case_three_aux_difference}
\end{align}
The first bracket is nonnegative by \eqref{eq:gap_induction_second}, and the second term is positive because $a_{k+1}>S_k$. Hence,
\begin{equation}
S_{k+1}\le2^k(r_{k+1}-G_{k+1}).
\end{equation}
This completes the coupled induction.

Rearranging \eqref{eq:gap_induction_first} gives, for every $k\ge2$,
\begin{equation}
G_k
\le
\bigl(1-2^{-(k-2)}\bigr)(S_k-r_k).
\label{eq:pointwise_sharp_gap_step}
\end{equation}
At $k=m$, $S_m(x)=p_X(x)$ and
$r_m(x)=\max_{j\in[m]}\pi_jp_j(x)$. Therefore, using \eqref{eq:gap_integral_of_Gm},
\begin{align}
\Delta_{\mathrm{ms}}
&=\int_{\mathcal X}G_m(x)\dif x\\
&\le
\bigl(1-2^{-(m-2)}\bigr)
\int_{\mathcal X}
\left(
 p_X(x)-\max_{j\in[m]}\pi_jp_j(x)
\right)\dif x\\
&=
\bigl(1-2^{-(m-2)}\bigr)\Bayes.
\label{eq:gap_bound_small_bayes_app}
\end{align}

\emph{Second pointwise bound.}
We next prove independently, again by induction, that
\begin{equation}
G_k\le\frac{k r_k-S_k}{2},
\qquad k\ge2.
\label{eq:gap_complementary_pointwise}
\end{equation}
For $k=2$, $G_2=0$ and $2r_2-S_2\ge0$, so the claim holds. Assume it holds for some $k\ge2$. We again use the three cases in \eqref{eq:gap_increment_three_cases}.

\emph{Case 1: $a_{k+1}\le r_k$.}
Since $g_{k+1}=0$ and $r_{k+1}=r_k$,
\begin{align}
G_{k+1}
&=G_k\\
&\le\frac{k r_k-S_k}{2}\\
&\le\frac{(k+1)r_k-S_k-a_{k+1}}{2}\\
&=\frac{(k+1)r_{k+1}-S_{k+1}}{2},
\end{align}
where the second inequality follows from $a_{k+1}\le r_k$.

\emph{Case 2: $r_k<a_{k+1}\le S_k$.}
Here, $g_{k+1}=a_{k+1}-r_k$ and $r_{k+1}=a_{k+1}$. Thus,
\begin{align}
G_{k+1}
&=G_k+a_{k+1}-r_k\\
&\le\frac{k r_k-S_k}{2}+a_{k+1}-r_k\\
&=\frac{(k-2)r_k+2a_{k+1}-S_k}{2}\\
&\le\frac{k a_{k+1}-S_k}{2}\\
&=\frac{(k+1)r_{k+1}-S_{k+1}}{2}.
\end{align}
The last inequality holds because
$(k-2)(a_{k+1}-r_k)\ge0$.

\emph{Case 3: $a_{k+1}>S_k$.}
Here, $g_{k+1}=S_k-r_k$ and $r_{k+1}=a_{k+1}$, yielding
\begin{align}
G_{k+1}
&=G_k+S_k-r_k\\
&\le\frac{k r_k-S_k}{2}+S_k-r_k\\
&=\frac{(k-2)r_k+S_k}{2}\\
&\le\frac{k a_{k+1}-S_k}{2}\\
&=\frac{(k+1)r_{k+1}-S_{k+1}}{2}.
\end{align}
For the last inequality, \eqref{eq:gap_sum_dominates_max} and $a_{k+1}>S_k$ give
\begin{equation}
(k-2)r_k+S_k
\le(k-1)S_k
<k a_{k+1}-S_k.
\end{equation}
This completes the second induction.

Taking $k=m$ in \eqref{eq:gap_complementary_pointwise} and using \eqref{eq:gap_integral_of_Gm}, we obtain
\begin{align}
\Delta_{\mathrm{ms}}
&\le
\frac12
\int_{\mathcal X}
\left(
 m\max_{j\in[m]}\pi_jp_j(x)-p_X(x)
\right)\dif x\\
&=
\frac{m(1-\Bayes)-1}{2},
\label{eq:gap_bound_random_label_app}
\end{align}
where we used
$\int_{\mathcal X}\max_j\pi_jp_j(x)\dif x=1-\Bayes$
and $\int_{\mathcal X}p_X(x)\dif x=1$.

The two inductions give two valid upper bounds on the same quantity, so we take their minimum:
\begin{equation}
0\le\Delta_{\mathrm{ms}}
\le
\min\left\{
\bigl(1-2^{-(m-2)}\bigr)\Bayes,
\frac{m(1-\Bayes)-1}{2}
\right\}
=\Delta^\star.
\end{equation}
The lower bound follows from $g_k\ge0$, and
$\Delta^\star\le\Bayes$ follows from
$0\le1-2^{-(m-2)}\le1$.
Equating the two branches gives
\begin{equation}
\Bayes^\star
=
\frac{m-1}{m+2\bigl(1-2^{-(m-2)}\bigr)},
\end{equation}
which is \eqref{eq:gap_bound_switch_point}. The first branch is nondecreasing in $\Bayes$ and the second is decreasing in $\Bayes$; hence, the first is the minimum for $0\le\Bayes\le\Bayes^\star$, and the second is the minimum for $\Bayes^\star\le\Bayes\le(m-1)/m$. This proves all claims in Theorem~\ref{thm:max_sum_gap_size}.
\hfill\qedsymbol

\section{Proof of Theorem~\ref{thm:multiclass_bias_variance}}
\label{app:proof_multiclass_bias_variance}
\noindent
\rev{Condition} \rev{throughout} \rev{on} \rev{the} \rev{auxiliary} \rev{set} \rev{of} \rev{samples} $\rev{\mathcal T_U}$\rev{.} \rev{The} \rev{ratio} \rev{estimates} \rev{are} \rev{then} \rev{fixed}\rev{,} \rev{and} \rev{the} \rev{evaluation} \rev{samples} \rev{remain} \rev{independent} \rev{with} \rev{their} \rev{specified} \rev{class} \rev{distributions}\rev{.} For each $\lambda\ge2$, define the population aggregate hinge
\begin{equation}
\tsum^{(\lambda)}(u)
\coloneqq
\pos{
\pi_\lambda-
\sum_{j=1}^{\lambda-1}\pi_j u_{j/\lambda}
}
\end{equation}
and the empirical average
\begin{equation}
A_\lambda
\coloneqq
\frac{1}{M_\lambda}
\sum_{i=1}^{M_\lambda}
\tsum^{(\lambda)}
\bigl(\widehat U_{\cdot/\lambda}(x_{\lambda,i})\bigr).
\label{eq:A_lambda_app}
\end{equation}
Then,
\begin{equation}
\widehat\varepsilon_{\mathrm{BOLT}}
=
1-\pi_1-
\sum_{\lambda=2}^{m}A_\lambda.
\label{eq:eps_A_form_app}
\end{equation}

We can show that the mapping $\tsum^{(\lambda)}$ is Lipschitz with respect to the $\ell_\infty$ norm. Since $z\mapsto\pos{z}$ is 1-Lipschitz, for any vectors $u$ and $v$,
\begin{align}
\left|
\tsum^{(\lambda)}(u)-
\tsum^{(\lambda)}(v)
\right|
&\le
\left|
\sum_{j=1}^{\lambda-1}\pi_j(u_{j/\lambda}-v_{j/\lambda})
\right|\\
&\le
\sum_{j=1}^{\lambda-1}\pi_j
|u_{j/\lambda}-v_{j/\lambda}|\\
&\le
\left(\sum_{j=1}^{\lambda-1}\pi_j\right)
\|u-v\|_\infty\\
&\le
\|u-v\|_\infty.
\label{eq:app_multiclass_lipschitz}
\end{align}

Let
\begin{equation}
\varepsilon_\Sigma
\coloneqq
1-\pi_1-
\sum_{\lambda=2}^{m}
\mathbb E_{X\sim P_\lambda}
\left[
\tsum^{(\lambda)}
\bigl(U_{\cdot/\lambda}(X)\bigr)
\right]
\end{equation}
denote the BOLT population risk at its pointwise minimizer. By Theorem~\ref{thm:multiclass_optimizer_tightness},
\begin{equation}
\varepsilon_\Sigma-\Bayes=\Delta_{\mathrm{ms}}.
\end{equation}
Applying the triangle inequality, we have
\begin{align}
\left|
\mathbb E[\widehat\varepsilon_{\mathrm{BOLT}}\rev{\mid}\rev{\mathcal T_U}]-\Bayes
\right|
&\le
\left|\varepsilon_\Sigma-\Bayes\right|
+
\left|
\mathbb E[\widehat\varepsilon_{\mathrm{BOLT}}\rev{\mid}\rev{\mathcal T_U}]-\varepsilon_\Sigma
\right|\\
&=
\Delta_{\mathrm{ms}}
+
\left|
\mathbb E[\widehat\varepsilon_{\mathrm{BOLT}}\rev{\mid}\rev{\mathcal T_U}]-\varepsilon_\Sigma
\right|.
\label{eq:app_bias_split}
\end{align}

For each $\lambda\ge2$, \rev{take} \rev{conditional} \rev{expectations} \rev{and} add and subtract the population hinge evaluated at $\widehat U_{\cdot/\lambda}$:
\begin{align}
&\mathbb E_{X\sim P_\lambda}
\left[
\tsum^{(\lambda)}(U_{\cdot/\lambda}(X))
\right]-\rev{\mathbb E}\rev{[}A_\lambda\rev{\mid}\rev{\mathcal T_U}\rev{]}\notag\\
&=
\mathbb E_{X\sim P_\lambda}
\left[
\tsum^{(\lambda)}(U_{\cdot/\lambda}(X))
-
\tsum^{(\lambda)}(\widehat U_{\cdot/\lambda}(X))
\,\middle|\,\rev{\mathcal T_U}\right]\notag\\
&\quad+
\left(
\mathbb E_{X\sim P_\lambda}
\left[
\tsum^{(\lambda)}(\widehat U_{\cdot/\lambda}(X))
\,\middle|\,\rev{\mathcal T_U}\right]
-\rev{\mathbb E}\rev{[}A_\lambda\rev{\mid}\rev{\mathcal T_U}\rev{]}
\right).
\label{eq:app_bias_decomposition}
\end{align}
The first term is at most $\varepsilon_0$ in absolute value by \eqref{eq:app_multiclass_lipschitz} and the assumption of Theorem~\ref{thm:multiclass_bias_variance}. The second term is \rev{zero}\rev{,} \rev{since} \rev{with} \rev{known} \rev{priors} \rev{and} \rev{independent} \rev{test} \rev{samples}\rev{,} $\rev{A}\rev{_\lambda}$ \rev{is} \rev{conditionally} \rev{unbiased} \rev{for} \rev{the} \rev{hinge} \rev{expectation} \rev{at} \rev{the} \rev{fitted} \rev{ratio} \rev{estimates}\rev{.} \rev{Therefore}\rev{,} \rev{the} \rev{test}-sample \rev{size} \rev{affects} \rev{variance} \rev{but} \rev{adds} \rev{no} \rev{term} \rev{to} the \rev{conditional} \rev{bias}. Summing over $\lambda$ \rev{gives}
\begin{equation}
\left|\rev{\mathbb E}\rev{[}\rev{\widehat\varepsilon}_{\rev{\mathrm{BOLT}}}\rev{\mid}\rev{\mathcal T_U}\rev{]}\rev{-}\rev{\varepsilon}\rev{_\Sigma}\right|
\rev{\le} \rev{(}\rev{m}\rev{-}\rev{1}\rev{)}\rev{\varepsilon}\rev{_0}\rev{.}
\end{equation}
\rev{Substituting} \rev{this} \rev{into} \eqref{eq:app_bias_split}\rev{,} \rev{we} \rev{obtain}
\begin{equation}
\left|\rev{\mathbb E}\rev{[}\widehat\varepsilon_{\mathrm{BOLT}}\rev{\mid}\rev{\mathcal T_U}\rev{]}\rev{-}\rev{\Bayes}\right|
\le
\Delta_{\mathrm{ms}}
+(m-1)\varepsilon_0.
\end{equation}

For the variance, each $A_\lambda$ in \eqref{eq:A_lambda_app} is an average of \rev{conditionally} \rev{independent} terms \rev{in} $\rev{[}\rev{0},\rev{\pi}\rev{_\lambda}\rev{]}$\rev{,} \rev{since} \rev{the} \rev{ratio} \rev{estimates} \rev{are} \rev{nonnegative}\rev{.} \rev{Hence}\rev{,}
\begin{equation}
\Var(A_\lambda\rev{\mid}\rev{\mathcal T_U})
\rev{\le}\rev{\frac{\pi_\lambda^2}{4M_\lambda}}
=\mathcal O(M_\lambda^{-1}).
\end{equation}
For arbitrary random variables $Z_1,\ldots,Z_q$, the \rev{conditional} Cauchy--Schwarz inequality gives
\begin{equation}
\Var\left(\sum_{j=1}^{q}Z_j\,\middle|\,\rev{\mathcal T_U}\right)
\le
q\sum_{j=1}^{q}\Var(Z_j\rev{\mid}\rev{\mathcal T_U}).
\end{equation}
Applying this bound to the $m-1$ terms $A_2,\ldots,A_m$, and noting that $\pi_1$ is fixed, we have
\begin{align}
\Var\!\parens{\widehat\varepsilon_{\mathrm{BOLT}}\rev{\mid}\rev{\mathcal T_U}}
&\rev{=}
\Var\left(\sum_{\lambda=2}^{m}A_\lambda\,\middle|\,\rev{\mathcal T_U}\right)\\
&\le
(m-1)
\sum_{\lambda=2}^{m}\mathcal O(M_\lambda^{-1})\\
&=
\mathcal O\!\parens{\frac{(m-1)^2}{M^\star}},
\end{align}
\rev{where} $\rev{M}\rev{^\star}\rev{=}\rev{\min}\{\rev{M}\rev{_1}\rev{,}\rev{\ldots}\rev{,}\rev{M}\rev{_m}\}$\rev{.} \rev{This} proves the theorem.
\hfill\qedsymbol

\bibliographystyle{IEEEtran_BOLT_revisions}
\bibliography{references}

\end{document}